\documentclass[11pt]{article}

\usepackage{amssymb,amsmath,amsthm,bbm}
\usepackage{verbatim,float,url,dsfont}
\usepackage{graphicx,subcaption,psfrag}
\usepackage{algorithm}
\usepackage[noend]{algorithmic}
\usepackage{mathtools,enumitem}
\usepackage{multirow}
\usepackage{ragged2e}
\usepackage{xr-hyper}
\usepackage{caption}
\usepackage{array}

\usepackage[utf8]{inputenc} 
\usepackage[T1]{fontenc}    
\usepackage{booktabs}       
\usepackage{nicefrac}         
\usepackage{microtype}      

\usepackage[dvipsnames]{xcolor} 
\definecolor{hunyuanblue}{HTML}{0052D9} 
\usepackage{pifont} 
\usepackage{etoc} 
\usepackage{tikz} 
\usepackage{pgfplots}  
\pgfplotsset{compat=1.16}  

\ifdefined\TimesFont 
\usepackage{times} 
\fi

\ifdefined\ParSkip 
\usepackage{parskip} 
\fi

\newtheorem*{assumption*}{\assumptionnumber}
\providecommand{\assumptionnumber}{}


\makeatletter
\newcommand*\rel@kern[1]{\kern#1\dimexpr\macc@kerna}
\newcommand*\widebar[1]{%
  \begingroup
  \def\mathaccent##1##2{%
    \rel@kern{0.8}%
    \overline{\rel@kern{-0.8}\macc@nucleus\rel@kern{0.2}}%
    \rel@kern{-0.2}%
  }%
  \macc@depth\@ne
  \let\math@bgroup\@empty \let\math@egroup\macc@set@skewchar
  \mathsurround\z@ \frozen@everymath{\mathgroup\macc@group\relax}%
  \macc@set@skewchar\relax
  \let\mathaccentV\macc@nested@a
  \macc@nested@a\relax111{#1}%
  \endgroup
}
\makeatother

\def\E{\mathbb{E}}

\def\ind#1{\mathds{1}\left\{#1\right\}}

\usepackage{hunyuan}

\usepackage[capitalize,noabbrev]{cleveref}
\usepackage{tabularx}
\usepackage[most]{tcolorbox}
\usepackage{wrapfig}
\usepackage{xspace}
\usepackage{fancyhdr}
\usepackage{xurl}
\usepackage{placeins}   
\usepackage{colortbl}   

\definecolor{ourshl}{HTML}{E4EFF8}   
\definecolor{grouphl}{HTML}{EAEAEA}  
\newcommand{\tbest}[1]{\textbf{#1}}
\newcommand{\tsecond}[1]{\underline{#1}}
\newcommand{\resulttablespacing}{\renewcommand{\arraystretch}{1.25}}
\newcolumntype{C}{>{\centering\arraybackslash}X}

\makeatletter
\long\def\@makecaption#1#2{%
  \vskip 10pt
  \setbox\@tempboxa\hbox{#1: #2}%
  \ifdim \wd\@tempboxa >\hsize
    \noindent #1: #2\par
  \else
    \hbox to\hsize{\hfil\box\@tempboxa\hfil}%
  \fi}
\makeatother

\hypersetup{
  colorlinks=true,
  linkcolor=hunyuanblue,
  citecolor=hunyuanblue,
  urlcolor=hunyuanblue,
}

\makeatletter
\def\section{\@startsection{section}{1}{\z@}{-0.24in}{0.10in}
             {\large\bf\raggedright\color{hunyuanblue}}}
\def\subsection{\@startsection{subsection}{2}{\z@}{-0.20in}{0.08in}
                {\normalsize\bf\raggedright\color{hunyuanblue}}}
\makeatother

\fancypagestyle{plain}{%
  \fancyhf{}%
  \fancyfoot[C]{\thepage}%
}

\fancypagestyle{firststyle}{%
  \fancyhf{}%
  \fancyhead[L]{\raisebox{-18.5mm}[0pt][0pt]{%
    \includegraphics[height=9mm]{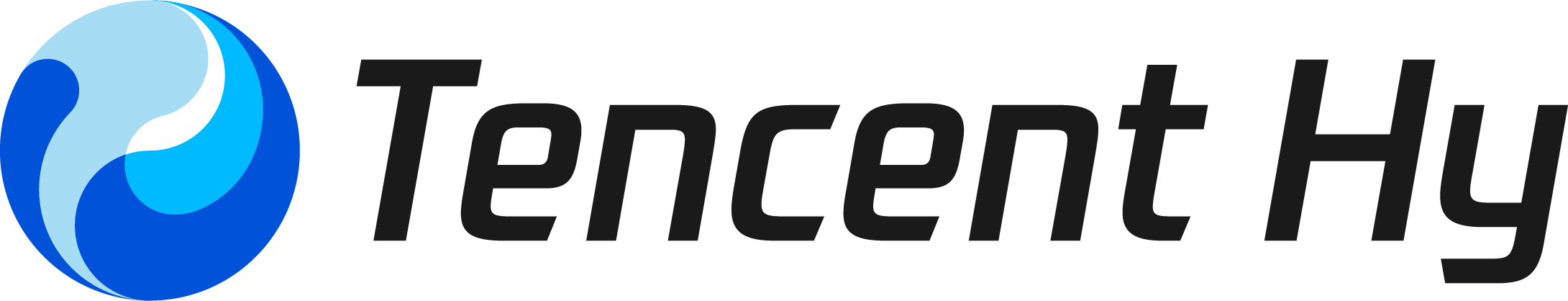}}}%
  \fancyfoot[C]{\thepage}%
}

\definecolor{abstractbg}{HTML}{F0F7FC}
\renewenvironment{abstract}{%
  \begin{tcolorbox}[%
    colframe=abstractbg, colback=abstractbg,
    boxrule=0pt, arc=2mm, enhanced,
    top=8pt, bottom=8pt, left=15pt, right=15pt,
    width=\textwidth]%
  \textbf{Abstract}\quad\ignorespaces%
}{\end{tcolorbox}}

\newcommand{\CPPO}{CPPO}
\newcommand{\CPPOSHORT}{\ifmmode\mathrm{CPPO}\else {CPPO}\fi}
\newcommand{\CPPOEXPAND}{\textbf{C}umulative \textbf{P}refix-divergence \textbf{P}olicy \textbf{O}ptimization}

\newcommand{\DPPO}{DPPO}
\newcommand{\TRM}{TRM}
\newcommand{\GRPO}{GRPO}
\newcommand{\PPO}{PPO}
\newcommand{\CISPO}{CISPO}
\newcommand{\MinPro}{MinPRO}
\newcommand{\SAPO}{S\kern-0.04em APO}
\newcommand{\DTV}{D_{\mathrm{TV}}}

\newcommand{\Db}{\delta_b}
\newcommand{\wmin}{w_{\min}}
\def\ind{\mathbf{1}}
\newcommand{\panelgraphics}[3]{%
  \begin{minipage}[t]{#1\linewidth}%
  \centering
  \includegraphics[width=\linewidth]{#3}%
  \end{minipage}%
}
\newcommand{\tightfigtop}{\vspace{-0.5em}}
\newsavebox{\figoneteaserbox}

\graphicspath{{figures/}}

\begin{document}

\thispagestyle{firststyle}
\vspace*{0.25cm}
{\color{hunyuanblue}\hrule height 0.6pt}
\vskip 3mm
{\centering
{\LARGE\bfseries Beyond Uniform Token-Level Trust Region in LLM\\[3pt]
Reinforcement Learning\par}
}
\vskip 4mm
{\color{hunyuanblue}\hrule height 0.6pt}
\vskip 6mm
\begin{center}
\textbf{Renjie Mao}$^{*}$ \quad
\textbf{Xiangxin Zhou}$^{*\dagger}$ \quad
\textbf{Lvfang Tao}$^{*\dagger}$ \quad
\textbf{Yixin Ding} \quad
\textbf{Yu Shi} \\[4pt]
\textbf{Yongguang Lin} \quad
\textbf{Yuheng Wu} \quad
\textbf{Honglin Zhu} \quad
\textbf{Qian Qiu} \quad
\textbf{Wenxi Zhu}$^{\dagger}$
\\[8pt]
Tencent Hunyuan
\\[6pt]
{\small $^{*}$Equal contribution \quad $^{\dagger}$Corresponding author}
\end{center}
\begin{abstract}
Reinforcement learning with verifiable rewards (RLVR) has become standard for improving LLM reasoning. 
However, existing \PPO{}-style trust-region mechanisms remain position-agnostic by enforcing
uniform thresholds across all tokens independently. This pointwise treatment conflicts with autoregressive
generation in two critical ways. First, uniform thresholds ignore autoregressive asymmetry. Early-stage deviations produce compounding sequence-level drift, causing
static thresholds to under-regulate early divergence
and excessively constrain late-stage exploration. Second, evaluating token-level divergence in isolation 
overlooks cumulative prefix drift, granting the same divergence allowance regardless of how far the 
conditioning history has already deviated from the rollout policy.
To address this limitation, we propose \CPPO{} (\emph{\CPPOEXPAND{}}), a token-level masking rule that aligns
updates with a finite-horizon policy-improvement bound via two coupled mechanisms. First, a position-weighted 
threshold imposes stricter limits at early positions whose effects persist longer, relaxing constraints 
for late-stage tokens. Second, a cumulative prefix budget tracks historical deviations, dynamically 
restricting further token-level deviation to prevent compounding errors along the prefix. Empirically, \CPPO{} 
enhances training stability and significantly improves reasoning accuracy across various model scales.

\vspace{0.3em} 
\noindent
{\small
\textbf{Correspondence:} \texttt{wenxizhu@tencent.com, ltao@pku.edu.cn, zhouxiangxin1998@gmail.com}
\\[2pt]
\textbf{Project Page:} \url{https://hunyuan-cppo.github.io}
\par} 
\end{abstract}

\begin{figure}[H]
\tightfigtop
\centering
\panelgraphics{0.64}{a}{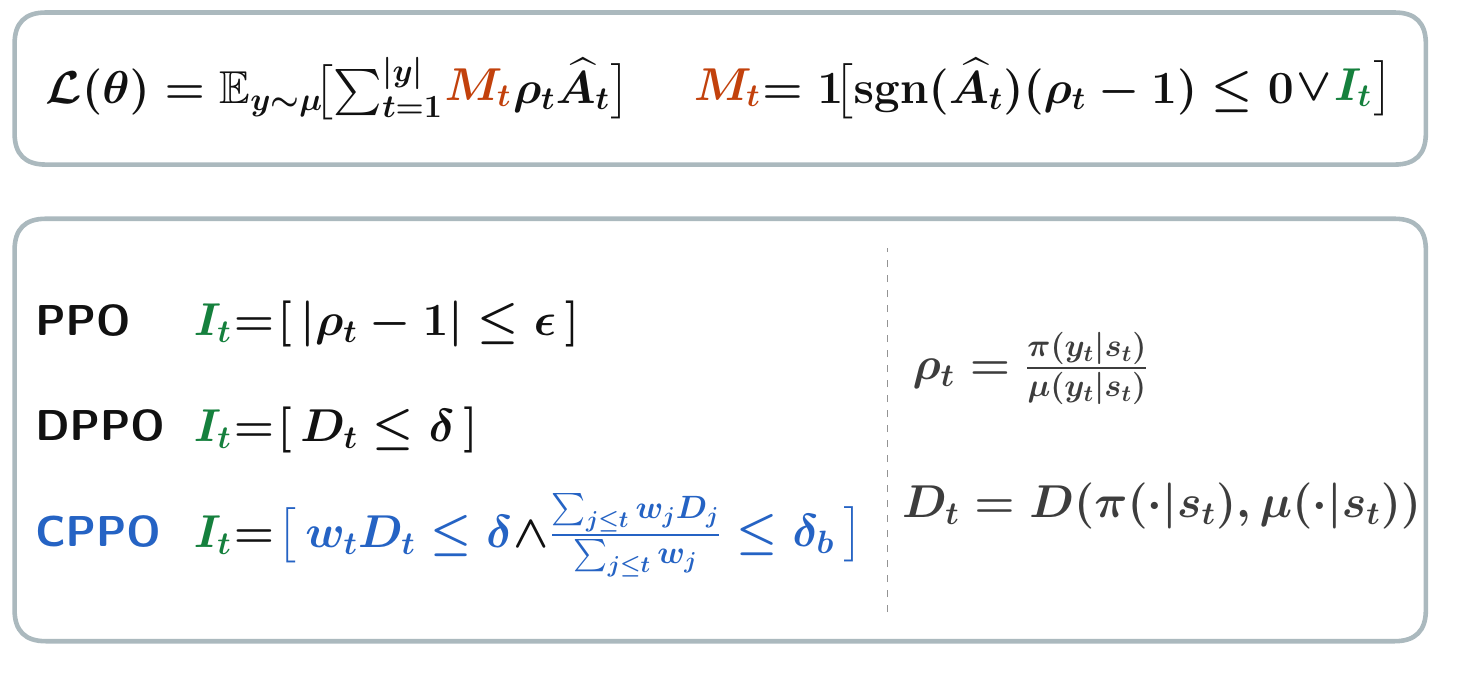}\hfill
\panelgraphics{0.36}{b}{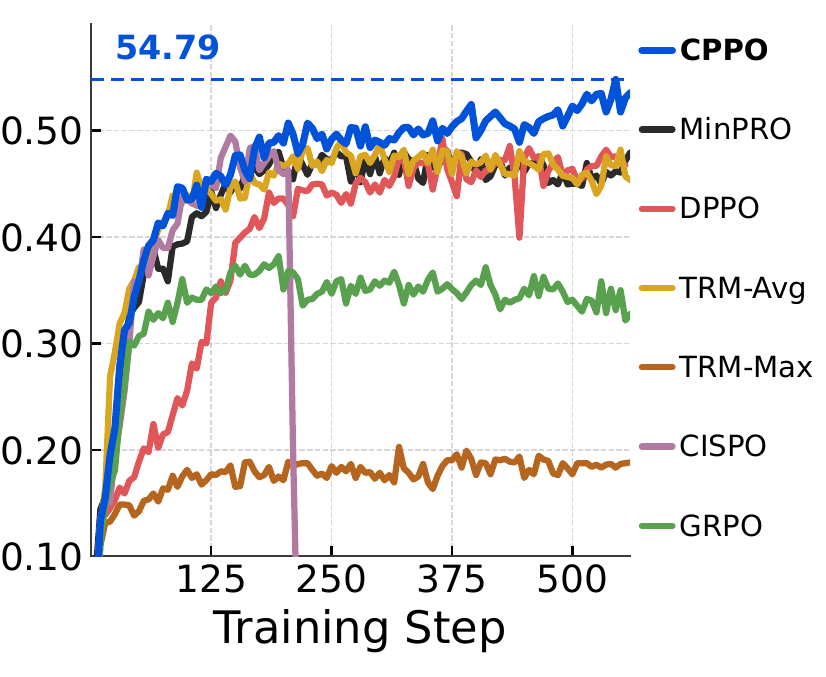}
\caption{Overview of \CPPO{}.
\textbf{Left}: token-level masking rules for
\PPO{}~\citep{schulman2017proximal}, \DPPO{}~\citep{qi2026rethinking}, and \CPPO{}.
\textbf{Right}: validation AIME24/25/26 $\mathrm{Avg}@16$ on
Qwen3-30B-A3B-Base. Full validation curves for the Base-model runs are shown in
Figure~\ref{fig:main-curves}.}
\label{fig:overview}
\end{figure}

\section{Introduction}
\label{sec:intro}

Reinforcement learning has become a standard tool for LLM
post-training, from preference-feedback alignment to
verifier-driven reasoning~\citep{DBLP:journals/corr/abs-2203-02155,rafailov2023direct,DBLP:journals/corr/abs-2402-03300,DBLP:journals/nature/GuoYZSWZXZMBZY025,yu2026dapo,liu2025understanding}.
In reinforcement learning with verifiable rewards (RLVR), the policy
generates responses, a verifier assigns a scalar reward, and the model
is updated with a \PPO{}/\GRPO{}-style token-level
objective~\citep{schulman2017proximal,DBLP:journals/corr/abs-2402-03300}.

A practical RLVR update is off-policy. Each batch of responses is
sampled from a fixed rollout policy $\mu$ and then reused for several
gradient steps, so the policy being optimized, $\pi$, steadily moves
away from the policy that generated the data. Unconstrained policy
optimization often leads to unstable updates and degraded reasoning
performance. Autoregressive generation further amplifies this divergence
because early token deviations alter the conditioning of all subsequent
steps.

To mitigate this drift, RLVR borrows the trust-region idea from
classical policy optimization. Trust Region Policy Optimization (TRPO)
constrains the divergence between successive policies and, in return,
guarantees monotonic
improvement~\citep{schulman2015trust,kakade2002approximately,achiam2017constrained,peters2010relative}.
Enforcing such a divergence exactly is expensive at LLM vocabulary
scale, so \PPO{} and \GRPO{} instead clip the likelihood ratio of the
sampled token~\citep{schulman2017proximal,DBLP:journals/corr/abs-2402-03300}.
For each token, this ratio is a single-sample Monte Carlo estimate of
the true divergence between $\mu$ and $\pi$; built from one sampled
token, it is noisy and has high
variance~\citep{DBLP:journals/corr/abs-1901-10314,DBLP:conf/uai/WangHT19,Engstrom2020Implementation,andrychowicz2021what}.
It is also poorly calibrated across the long-tailed LLM vocabulary: a
rare token can yield a large ratio and be over-penalized, while a
high-probability token can move substantial probability mass under a
ratio close to one and go under-penalized~\citep{qi2026rethinking}.
\DPPO{}~\citep{qi2026rethinking} replaces this estimate with a direct
measure of policy divergence. At each token it constrains the
total-variation (TV) divergence $D_t$ between the rollout policy $\mu$
and the target policy $\pi$, which reflects the actual change in the
next-token distribution rather than a single sampled outcome.

While these methods refine the measurement of token-level policy divergence,
they apply a uniform, position-agnostic threshold across all steps.
This pointwise constraint conflicts with the autoregressive factorization of LLMs
by overlooking two structural properties of the generation process:
First, \textbf{uniform thresholds ignore the autoregressive asymmetry in error propagation.}
Because early tokens condition the entire subsequent generation, an identical
token-level divergence at earlier positions induces a larger sequence-level distribution shift.
By assigning a static threshold across all positions, 
uniform constraints inherently underestimate early-stage deviations,
while overly restricting late-stage exploration.
Second, \textbf{evaluating token-level divergence in isolation ignores cumulative prefix divergence.}
In LLM reinforcement learning, the state at step $t$ is defined by the generated prefix $s_t = (x, y_{<t})$.
If individual token deviations accumulate along this prefix, the model is effectively generating from a highly off-policy state.
When the conditioning prefix has already drifted significantly from the rollout policy, 
any further deviation at the current token carries compounding risk.
A uniform threshold fails to account for this historical context, 
granting the exact same room for divergence regardless of how far the prefix has already drifted.
Consequently, the trust region should dynamically tighten as off-policy drift accumulates along the prefix,
with the permissible divergence threshold for subsequent tokens proportionally reduced to prevent sequence-level collapse.

To address this limitation, we propose \CPPO{} (\emph{\CPPOEXPAND{}}),
a novel trust-region mechanism that directly aligns token-level updates with the autoregressive structure of LLMs. 
While recent advances focus on \textit{how} token-level divergence is measured, 
\CPPO{} addresses \textit{where} and \textit{how much} policy deviation is permitted to accumulate along a trajectory.
Specifically, rather than enforcing pointwise divergence limits in isolation,
\CPPO{} regularizes the policy update through two coupled constraints designed to mitigate the aforementioned structural mismatches. 
First, to account for autoregressive asymmetry, we introduce a \textbf{position-weighted token-level threshold.}
This mechanism enforces relatively conservative divergence limits at early positions
where deviations cascade into substantial sequence-level distribution shifts,
while relaxing the constraints for late-stage tokens to preserve exploration.
Second, to prevent compounding errors, we establish a \textbf{cumulative prefix budget.}
By tracking and budgeting the weighted average of divergences along the generated prefix, 
\CPPO{} dynamically restricts further updates once the historical context has drifted significantly from the rollout policy.
Theoretically, this dual-constraint approach explicitly bounds the finite-horizon error propagation, 
yielding a provably tighter policy-improvement bound compared to uniform, position-agnostic thresholding. 
Empirically, this structured allocation of the divergence budget significantly enhances training stability 
and reasoning accuracy across various model scales. 
Figure~\ref{fig:overview} summarizes the token-level masking conditions
and representative Qwen3-30B-A3B-Base validation curves.

This paper makes the following contributions:
\begin{itemize}[leftmargin=*]
    \item We formalize how token position enters the finite-horizon error bound: early token-level policy shifts affect longer suffixes and contribute more, explaining why uniform token-level divergence thresholds are loose for long responses.
    \item We introduce \CPPO{}, a token-level mask that follows the prefix-to-suffix generation order and constrains token-level divergence with a position-weighted token-level constraint and a cumulative prefix budget.
    \item We integrate \CPPO{} as a drop-in token mask and evaluate it under matched RLVR settings. Across varying model sizes, \CPPO{} obtains the best AIME24/25/26 average scores, with ablations supporting both constraints.
\end{itemize}

\section{Preliminaries}
\label{sec:theory}

Reinforcement learning for LLMs is a finite-horizon sequential decision
problem. Given a prompt $x\sim\mathcal D$, a policy $\pi$ generates a
response $y=(y_1,\dots,y_T)$ one token at a time. At step $t$ the state is
the prompt together with the tokens generated so far, $s_t=(x,y_{<t})$,
the action is the next token $y_t$, and the policy defines a conditional
distribution $\pi(y_t\mid s_t)$ over the vocabulary. The response
probability factorizes autoregressively as
$\pi(y\mid x)=\prod_{t=1}^T\pi(y_t\mid s_t)$. After the full response is
produced, a verifier returns a scalar reward $R(x,y)$ with
$|R(x,y)|\le\xi$, and the training objective is
$J(\pi)=\E_{x,y\sim\pi}[R(x,y)]$. An RLVR update is off-policy: responses
are drawn from a fixed rollout policy $\mu$ and reused to optimize a
target policy $\pi$ under common support.

Trust-region methods control how far the target policy $\pi$ may move from
the rollout policy $\mu$. Trust Region Policy Optimization (TRPO)
maximizes a surrogate objective subject to an explicit constraint on the
policy divergence, which guarantees monotonic improvement as long as
$\pi$ stays close to
$\mu$~\citep{schulman2015trust,kakade2002approximately,achiam2017constrained}.
The trust-region rules used in LLM RL are heuristic approximations of this
constrained update: \PPO{} and \GRPO{} replace the divergence constraint
with clipping of the sampled likelihood
ratio~\citep{schulman2017proximal,DBLP:journals/corr/abs-2402-03300},
\DPPO{} replaces the sampled ratio with a direct measure of the
token-level divergence~\citep{qi2026rethinking}, and \TRM{} applies a
similar divergence test at the sequence level~\citep{li2025trust}. This
section sets up the finite-horizon surrogate these methods share, the
token-level trust-region rules they impose, and how the token-level
divergence is computed in practice. Full proofs and extensions are in
Appendix~\ref{app:fulltheory}.

\subsection{Finite-horizon token-level surrogate}
\label{sec:identity}

We fix the rollout policy $\mu$ and optimize the target policy $\pi$ under
common support. Let
\[
\rho_t:=\frac{\pi(y_t\mid s_t)}{\mu(y_t\mid s_t)},\qquad
\rho_{a:b}:=\prod_{j=a}^b\rho_j,
\]
with $\rho_{T+1:T}=1$.

Following \citet{qi2026rethinking}, autoregressive factorization gives
the exact finite-horizon performance difference identity
(Lemma~\ref{lem:exact-id}, proved in Appendix~\ref{app:fulltheory})
\begin{equation}
J(\pi)-J(\mu)=L'_\mu(\pi)-\Delta(\mu,\pi),
\label{eq:exact-id}
\end{equation}
where
\[
\begin{aligned}
L'_\mu(\pi)
&:=\E_\mu\!\left[R(x,y)\sum_{t=1}^T(\rho_t-1)\right],\\
\Delta(\mu,\pi)
&:=\E_\mu\!\left[
R(x,y)\sum_{t=1}^T(\rho_t-1)(1-\rho_{t+1:T})
\right].
\end{aligned}
\]
The factor $\rho_{t+1:T}$ is the likelihood ratio of the future
tokens $y_{t+1:T}$ under $\pi$ relative to $\mu$, conditioned on
the prefix up to token $t$. Reverse telescoping gives the exact
corrected objective
$\E_\mu[R(x,y)\sum_t(\rho_t-1)\rho_{t+1:T}]$. The token-level
surrogate $L'_\mu$ sets this suffix correction to one. Trust-region
constraints are therefore tied to the surrogate itself because they
control the approximation error $|\Delta(\mu,\pi)|$ induced by
dropping the future likelihood-ratio correction.

\subsection{From sampled-ratio clipping to token-level divergence}
\label{sec:token-tr}

The clipped \PPO{}-style objective used by \PPO{} and \GRPO{} is
the standard practical implementation of the token-level surrogate. For
token advantage $\hat A_t$ and clip range $\epsilon$, the clipped
token objective is
\[
\widehat L_\mu^{\mathrm{\PPO}}(\pi)
=\E_\mu\!\left[
\sum_{t=1}^T
\min\!\left(
\rho_t\hat A_t,\,
\operatorname{clip}(\rho_t,1-\epsilon,1+\epsilon)\hat A_t
\right)
\right].
\]
Equivalently, its one-sided clipping rule can be written as the
sampled-ratio mask
\begin{equation}
M_t^{\mathrm{\PPO}}
=\ind\!\left[
\hat A_t(\rho_t-1)\le 0
\;\;\vee\;\;
|\rho_t-1|\le\epsilon
\right].
\label{eq:ppo-mask}
\end{equation}
This mask makes explicit the trust-region criterion implicit in the
clipped token objective. The decision is made from the ratio of the
sampled token; \GRPO{} changes how the advantage is estimated, but uses
the same signed clipping asymmetry.

The sampled ratio is a one-sample view of the token-level distributional
change. For each state $s_t$,
\[
\DTV(\mu(\cdot\mid s_t),\pi(\cdot\mid s_t))
=\frac12\,\E_{y_t\sim\mu(\cdot\mid s_t)}|\rho_t-1|.
\]
\DPPO{} replaces the sampled-ratio test in
Equation~\eqref{eq:ppo-mask} with a direct measure of the token-level
divergence $D_t=D(\mu(\cdot\mid s_t),\pi(\cdot\mid s_t))$:
\begin{equation}
M_t^{\mathrm{\DPPO}}
=\ind\!\left[
\hat A_t(\rho_t-1)\le 0
\;\;\vee\;\;
D_t\le\delta
\right].
\label{eq:dppo-mask}
\end{equation}
This change replaces the sampled-ratio trust-region criterion with a
distributional next-token divergence criterion, but it still assigns a
uniform token-level threshold $\delta$ across token positions.

\subsection{Token-level divergence approximation}
\label{sec:divergence-score}

A \DPPO{}-style mask needs the token-level divergence $D_t$ between the
next-token distributions $\mu(\cdot\mid s_t)$ and $\pi(\cdot\mid s_t)$,
whose exact value is the total variation
$D_t=\DTV(\mu(\cdot\mid s_t),\pi(\cdot\mid s_t))$. Evaluating this over the
full vocabulary at every token is computationally prohibitive at LLM scale, so we
follow the \DPPO{} Top-$K$ reduced-TV approximation~\citep{qi2026rethinking}
($K{=}20$; construction in Appendix~\ref{app:exp-details}). All token-level
trust-region methods in our experiments use this same approximation and a
per-model threshold scale, so that any difference between them reflects the
trust-region rule rather than the divergence approximation. By Pinsker's
inequality $\DTV\le\sqrt{D_{\mathrm{KL}}/2}$, the construction and the
bound in Section~\ref{sec:method} cover a KL-based constraint as well.

\subsection{Limitations of uniform token-level thresholds}
\label{sec:uniform-limits}

Although existing methods differ in their token-level divergence metrics,
they universally apply a static threshold $\delta$ across all positions.
This uniform constraint is globally suboptimal due to two properties of
autoregressive generation.

First, \textbf{a uniform threshold ignores how a token-level divergence
propagates along the response.} The surrogate error $\Delta(\mu,\pi)$ that
a trust region must control comes from the dropped suffix correction
$1-\rho_{t+1:T}$ (Section~\ref{sec:identity}), so a divergence introduced
at token $t$ enters the conditioning of every token sampled after it. Early deviations affect long suffixes, whereas late deviations have
minimal downstream impact. A uniform threshold fails to account for this
asymmetry.

Second, \textbf{a uniform threshold ignores how divergence accumulates
along the prefix.} Because per-token deviations accumulate within the
historical context $s_t=(x,y_{<t})$, a sequence of locally bounded steps
can still shift the sampling prefix far from the rollout policy
distribution. A uniform token-level
threshold provides a constant divergence budget regardless of prior prefix
drift.

These two limitations motivate \CPPO{} (Section~\ref{sec:method}). \CPPO{}
retains the standard token-level divergence but revises how the permitted
divergence budget is distributed along the response via two novel mechanisms.
To match the divergence
allowed at a position to how far that divergence can propagate, \CPPO{}
replaces the uniform threshold with a \textbf{position-weighted threshold}
that is tighter at early positions and looser at late ones. To bound accumulated divergence, \CPPO{} introduces a \textbf{cumulative prefix
budget}, which caps the weighted average of token-level divergences over every
response prefix. Unlike a per-token threshold, this cap binds
on the prefix as a whole, so once earlier tokens have driven the prefix
average up, the divergence allowed at the next token is reduced
accordingly. The trust region therefore tightens dynamically as the prefix
drifts. The next section turns these two mechanisms into a concrete
token-level mask.

\section{\CPPO{}}
\label{sec:method}

Section~\ref{sec:uniform-limits} motivated two mechanisms for distributing
the permitted divergence along a response: a position-weighted token-level
threshold and a cumulative prefix budget. Before constructing the mask, we
make precise the finite-horizon bound that these mechanisms are designed to
improve. Section~\ref{sec:residual-bound} quantifies how a token-level
divergence propagates through the remaining response, and
Section~\ref{sec:thm} states the resulting policy-improvement bound.
Sections~\ref{sec:feasible} and~\ref{sec:mask} then turn the bound into the
\CPPO{} token-level mask. Throughout, the policy loss keeps the usual
token-level \PPO{}-style ratio-advantage form with \GRPO{} group-relative
advantages in our experiments; the only change is the masking decision for
update terms that move the sampled-token ratio farther from one.

\subsection{Autoregressive Asymmetry in Error Propagation}
\label{sec:residual-bound}

We first quantify the autoregressive asymmetry: how the position of a token-level divergence dictates its downstream impact on the
surrogate residual $\Delta(\mu,\pi)$ from Section~\ref{sec:identity}. Let
$D_t:=\DTV(\mu(\cdot\mid s_t),\pi(\cdot\mid s_t))$ and
$u_t:=\E_{s_t\sim d_t^\mu}[D_t]$ be the expected token-level divergence at
position $t$. 
Suppose a position-specific divergence limit $D_t\le\ell_t$, is enforced.
A maximal-coupling argument on the suffix likelihood ratio
(Lemma~\ref{lem:suffix-tv}) gives a bound on the surrogate residual
(Proposition~\ref{prop:residual-bound}, proved in
Appendix~\ref{app:fulltheory}):
\begin{equation}
|\Delta(\mu,\pi)|
\;\le\;
4\xi\sum_{t=1}^{T-1}u_t\sum_{j=t+1}^T\ell_j
\;\le\;
\sum_{t=1}^{T-1}\lambda_tu_t,
\qquad
\lambda_t:=4\xi\bar\ell\,(T-t),\quad
\bar\ell:=\max_j\ell_j ,
\label{eq:residual-bound}
\end{equation}
where $\xi$ is the absolute bound on the reward ($|R(x,y)| \le \xi$). Equation~\eqref{eq:residual-bound} makes the first limitation from
Section~\ref{sec:uniform-limits} quantitative. The coefficient
$\lambda_t=4\xi\bar\ell(T-t)$ attached to the expected token-level
divergence $u_t$ grows linearly with the remaining horizon $T-t$.
In sequential decision making, this linear dependence formalizes the
error propagation and compounding covariate shift inherent in autoregressive generation. An
early policy deviation shifts the distribution of generated prefixes,
skewing the conditioning context for all subsequent tokens and
accumulating its surrogate residual error over the entire suffix. 
This reveals a fundamental asymmetry: early tokens act as critical branching points with long-term consequences, whereas late tokens have minimal downstream impact. 
A uniform token-level threshold $D_t\le\delta$ ignores this profile entirely.
Consequently, it under-penalizes early deviations (which carry the
largest error propagation multipliers) while overly restricting late-stage
exploration, where divergence has minimal effect on future conditioning.

This autoregressive asymmetry, however, presents a structural opportunity: Because the error propagation multiplier $\lambda_t$ scales linearly with the remaining horizon $T-t$, we introduce a monotonic position weight $w_t$ designed to mirror this exact trajectory. By aligning the position weight with the intrinsic error propagation profile, the divergence constraint naturally relaxes as generation progresses. This connection between the surrogate bound and the position weight is formalized below.

\subsection{Prefix-constrained improvement bound}
\label{sec:thm}

\CPPO{} constrains token-level divergence through weighted prefix
averages. Let $w_t>0$ be a position weight, $c_t>0$ a weighted
token-level divergence threshold, and $\Db>0$ a prefix-average threshold.
Define
\[
P_m:=\sum_{j=1}^m w_jD_j,\qquad
W_m:=\sum_{j=1}^m w_j .
\]
The theorem uses the following token-level and prefix constraints along
the rollout:
\begin{equation}
w_tD_t\le c_t\quad(t=1,\dots,T),
\qquad
P_m\le\Db W_m\quad(m=1,\dots,T-1).
\label{eq:prefix-feasible-main}
\end{equation}
The theorem is stated in weighted form because the implementation
constrains $w_tD_t$ directly. For a fixed threshold on $w_tD_t$, positions
with larger $w_t$ allow a smaller value of $D_t$, while the prefix
constraint limits the weighted average policy deviation on every prefix. Set
$\ell_t:=c_t/w_t$ and $\bar\ell:=\max_t\ell_t$.

\begin{theorem}[\CPPO{} policy-improvement bound]
\label{thm:pact-main}
Suppose~\eqref{eq:prefix-feasible-main} holds $\mu$-a.s.\ along
the rollout. Let $\lambda_t=4\xi\bar\ell(T-t)$ and assume the ratio of the error propagation penalty to the position weight, $r_t:=\lambda_t/w_t$, is non-increasing in $t=1,\dots,T-1$. Then
\begin{equation}
J(\pi)-J(\mu)\;\ge\;L'_\mu(\pi)
-2\xi T(T-1)\,\bar\ell\,\Db .
\label{eq:pact-main}
\end{equation}
In particular, for a constant threshold $c_t\equiv\delta$, and a weight floor $\wmin$ where $w_t\in[\wmin,1]$, then
\begin{equation}
J(\pi)-J(\mu)\;\ge\;L'_\mu(\pi)
-2\xi T(T-1)\,\delta\,\frac{\Db}{\wmin}.
\label{eq:pact-main-impl}
\end{equation}
\end{theorem}

The proof is deferred to Appendix~\ref{app:fulltheory}; it combines the
performance difference identity (Lemma~\ref{lem:exact-id}) with the
remaining-horizon residual bound (Proposition~\ref{prop:residual-bound})
and an Abel-summation step over the prefix constraints.

To facilitate comparison, we separate the current-token divergence term
from the suffix-divergence terms in Equation~\eqref{eq:residual-bound}. The
coefficient $\lambda_t=4\xi\bar\ell(T-t)$ contains the suffix
threshold $\bar\ell$, while $u_t=\E[D_t]$ is the expected token-level
divergence at position $t$:
\[
|\Delta|\;\le\;\sum_{t<T}\lambda_tu_t,
\qquad
\lambda_t=4\xi\bar\ell(T-t).
\]
A position-independent token-level divergence method with $D_t\le\delta$
controls both factors pointwise,
\[
\bar\ell\le\delta,\qquad u_t\le\delta,
\]
and therefore gives
\[
|\Delta|
\le 4\xi\delta^2\sum_{t=1}^{T-1}(T-t)
=2\xi T(T-1)\delta^2 .
\]
\CPPO{} keeps the same suffix-divergence factor $\bar\ell$. Its prefix
constraints change how the expected token-level divergences are bounded:
instead of using the pointwise implication $u_t\le\delta$, Abel
summation gives
\[
\sum_{t<T}\lambda_tu_t
\le
\Db\sum_{t<T}\lambda_t
=2\xi T(T-1)\bar\ell\Db .
\]
Thus the prefix constraints replace the pointwise token-level divergence factor
$\delta$ by the prefix-average threshold $\Db$.

\noindent\textbf{Uniform token-level threshold.} For the clean comparison, both
methods use a uniform token-level threshold $D_t\le\delta$
(Corollary~\ref{cor:uniform}).
This corresponds to setting $c_t=w_t\delta$ for \CPPO{}, so
$\bar\ell=\delta$. The two residual constants are then
\[
C_{\mathrm{uniform}}
=2\xi T(T-1)\delta^2,
\qquad
C_{\mathrm{\CPPOSHORT}}
=2\xi T(T-1)\delta\Db,
\]
and hence
\[
\frac{C_{\mathrm{\CPPOSHORT{}}}}{C_{\mathrm{uniform}}}
=\frac{\Db}{\delta}.
\]
This comparison fixes the threshold value $\delta$ across the two methods.
The improvement comes from the prefix constraints, which prevent
many early prefixes with large remaining-horizon coefficients from
all reaching the token-level threshold. The bound improves when
$\Db<\delta$.

Replacing the purely pointwise bound with the weighted prefix-average
threshold $\Db$ directly mitigates worst-case error accumulation. A
pointwise constraint permits the policy to saturate the divergence
budget at every consecutive step. Because the state in autoregressive
generation is the historical prefix $s_t=(x,y_{<t})$, such worst-case
accumulation drives the state visitation under $\pi$ far from that
under $\mu$, which directly inflates the surrogate residual
$|\Delta(\mu,\pi)|$. Abel summation refactors the residual sum in terms
of the prefix sums $P_m$. Consequently, constraining weighted prefix
averages rather than pointwise divergences tightens the bound by
precluding the worst-case accumulation patterns permitted by pointwise
constraints. The factor $\Db/\delta$ above
quantifies how much of that pointwise looseness the prefix budget
closes.

\noindent\textbf{Position-dependent token-level threshold.} The implementation in
Section~\ref{sec:feasible} constrains the weighted divergence by
$w_tD_t\le\delta$, which is equivalent to a position-dependent token-level
threshold $D_t\le\delta/w_t$: the threshold equals $\delta$ at the
beginning of the response and is relaxed toward the end
(Corollary~\ref{cor:impl}). Since $w_t\ge\wmin$,
\[
\bar\ell=\max_t\frac{\delta}{w_t}
\le\frac{\delta}{\wmin}.
\]
Combining this with the prefix part of the theorem gives
\[
C_{\mathrm{\CPPOSHORT{}}}
\le
2\xi T(T-1)\frac{\delta}{\wmin}\Db,
\qquad
\frac{C_{\mathrm{\CPPOSHORT{}}}}{C_{\mathrm{uniform}}}
\le
\frac{\Db}{\delta\wmin}.
\]
The factor $1/\wmin$ is the cost of relaxing late-token thresholds.
Under this implementation parameterization, the bound is tighter
than the position-independent threshold $D_t\le\delta$ when
$\Db<\delta\,\wmin$.

\noindent\textbf{From theoretical bounds to trust-region mechanisms.}
Theorem~\ref{thm:pact-main} establishes that tightening the surrogate
residual requires transitioning from a loose, position-agnostic
dependence on $\delta$ to a stricter bound governed by the weighted
prefix-average threshold $\Db$. This transition requires
two structural modifications to the policy update. First, the
remaining-horizon coefficient ($\lambda_t\propto T-t$) formally
captures autoregressive asymmetry. Early policy deviations propagate
over longer suffixes and therefore demand tighter regulation. We
address this by introducing a monotonically decreasing position weight
$w_t$, which enforces stricter limits initially and naturally relaxes
them as generation progresses. Second, the analytical reliance on
$\Db$ demonstrates that policy drift must be constrained cumulatively
across the conditioning history rather than evaluating next-token
divergences in isolation. We implement this via a
dynamic cumulative prefix budget. The following subsections translate
these theoretical properties into the concrete token-level masking
rule implemented by \CPPO{}.

\subsection{Prefix constraints}
\label{sec:feasible}

Theorem~\ref{thm:pact-main} requires the cumulative budget
$P_m\le\Db W_m$ to hold at every intermediate prefix
($m=1,\ldots,T-1$), not only at the full response. Enforcing this
average bound exclusively at the sequence terminal (e.g., \TRM{}-Avg)
creates a loose constraint. It permits excessive deviations at early
positions, provided later tokens mathematically offset them to satisfy
the overall response average.
To enforce the prefix constraint at every $m$ during training, when
only the prefix produced so far is available at token $t$, we turn
the cumulative requirement into a single threshold on the weighted
token-level divergence $w_tD_t$ that depends only on the preceding
tokens.

Recall the token-level divergence $D_t$ from
Section~\ref{sec:divergence-score}. Define the weighted prefix divergence
and cumulative weight by
\[
S_t:=\sum_{j=1}^t w_jD_j,\qquad
W_t:=\sum_{j=1}^t w_j,\qquad S_0=W_0:=0 ,
\]
where $S_t$ is the running form of the prefix sum $P_t$ in
Theorem~\ref{thm:pact-main}. Given a token-level threshold scale $\delta$,
a prefix-average threshold $\Db$ (the maximum weighted average allowed over
each prefix), and positive weights $\{w_t\}_{t=1}^T$, \CPPO{} keeps an
update term that moves the sampled-token ratio farther from one
only when
\begin{equation}
w_tD_t\;\le\;
c_t^{\mathrm{\CPPOSHORT{}}}
:=\min\!\left\{
\delta,\,
\delta+\Db W_{t-1}-S_{t-1}
\right\},
\qquad t=1,\dots,T .
\label{eq:pact-feasible}
\end{equation}
This effective threshold enforces the weighted token-level condition
$w_tD_t\le\delta$ together with the cumulative prefix condition
$S_t\le\delta+\Db W_{t-1}$. The extra $\delta$ term gives the first token
the full token-level threshold prior to any prefix accumulation.
At $t=1$, $W_0=S_0=0$ and the effective
prefix threshold is $\delta$. For later tokens, the effective
prefix threshold is above the token-level threshold when the preceding
weighted average $S_{t-1}/W_{t-1}$ is below $\Db$, so the token-level
threshold remains active. Once preceding deviations make this prefix
average exceed $\Db$, the prefix-adjusted term
$\delta+\Db W_{t-1}-S_{t-1}$ falls below $\delta$ and reduces the
allowed weighted divergence at the current token. The effective
threshold therefore tightens dynamically as prefix divergence accumulates, instead
of granting every token the same allowance.

The implementation uses the decreasing linear schedule
\begin{equation}
w_t=1-\frac{1-\wmin}{T-1}(t-1),\qquad t=1,\dots,T,\quad
w_t\in[\wmin,1].
\label{eq:linear-w}
\end{equation}
Because the constraint is applied to $w_tD_t$, the allowed token-level
divergence at position $t$ is $D_t\le\delta/w_t$. This threshold starts at
$\delta$ at the beginning of the response and relaxes to
$\delta/\wmin$ near the end. The weights parameterize the token-level
and prefix constraints rather than defining a new divergence measure or
loss term. Early positions therefore allow smaller values of $D_t$,
while later positions are relaxed where the remaining suffix is shorter.
The schedule satisfies the monotonicity condition required by
Theorem~\ref{thm:pact-main}, namely that $(T-t)/w_t$ is
non-increasing in $t$ (Proposition~\ref{prop:linear-mono} in
Appendix~\ref{app:fulltheory}). The scalar
$\delta$ controls the token-level threshold scale, and $\Db$
controls the weighted average allowed at every prefix.

Figure~\ref{fig:position-diagnostic} places this schedule next to
an offline diagnostic from
Qwen3-30B-A3B rollouts after 200 \GRPO{} updates. For each token
position, we compute
$|\pi_\theta(a_t\mid s_t)-\mu(a_t\mid s_t)|$, where
$\pi_\theta$ is the post-update training policy and $\mu$ is the
rollout policy. The mean and top-$10\%$ tail views show that
policy deviation varies with token position. The right panel shows
the induced token-level threshold $\delta/w_t$
used by \CPPO{}.

\begin{figure}[!ht]
\tightfigtop
\centering
\panelgraphics{0.32}{a}{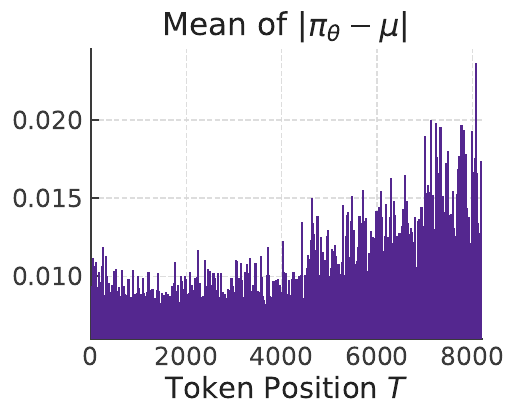}\hfill
\panelgraphics{0.32}{b}{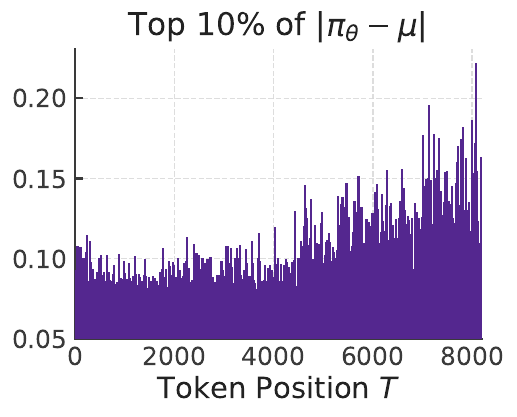}\hfill
\panelgraphics{0.32}{c}{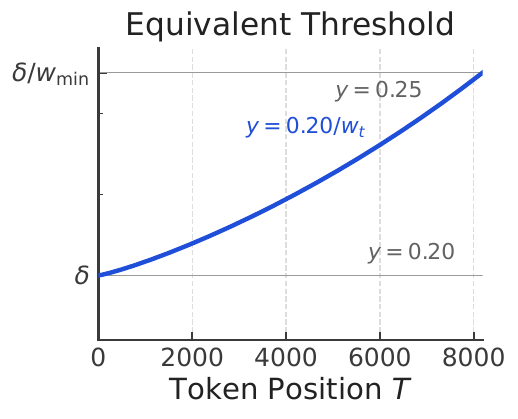}
\caption{Position-conditioned policy deviation and the
position-dependent token-level threshold.
\textbf{Left}: maximum of bin-level mean absolute probability
deviation across token positions.
\textbf{Middle}: the corresponding top-$10\%$ tail view.
\textbf{Right}: the equivalent token-level threshold $\delta/w_t$
induced by the decreasing $w_t$ schedule.}
\label{fig:position-diagnostic}
\end{figure}

Figure~\ref{fig:prefix-constraint} illustrates
Equation~\eqref{eq:pact-feasible} on a schematic token window. The
orange curve is the prefix-adjusted threshold
$\delta+\Db W_{t-1}-S_{t-1}$, and the green curve is the effective
threshold after taking the minimum with the token-level threshold
$\delta$. The blue shaded region marks tokens for which the token-level
threshold is active; the orange shaded region marks tokens for which the
prefix constraint is active. Because the implementation includes the
initial prefix slack in Equation~\eqref{eq:pact-feasible}, the effective
threshold starts at $\delta$. Low initial weighted deviations keep the
token-level threshold active, while larger accumulated deviations later
reduce the effective threshold below $\delta$. A token can then satisfy
the token-level threshold but exceed the effective threshold, in which
case it is masked by the prefix constraint.

The initial slack is an implementation detail rather than part of
the formal theoretical bound. It changes the prefix inequality from
$S_t\le\Db W_t$ to $S_t\le\delta+\Db W_{t-1}\le\Db W_t+\delta$,
adding only a lower-order $O(T\bar\ell\delta)$ term to the Abel
bound in Theorem~\ref{thm:pact-main};
Proposition~\ref{prop:initial-prefix-slack} in
Appendix~\ref{app:fulltheory} gives the calculation.

\begin{figure}[!ht]
\tightfigtop
\centering
\includegraphics[width=\linewidth]{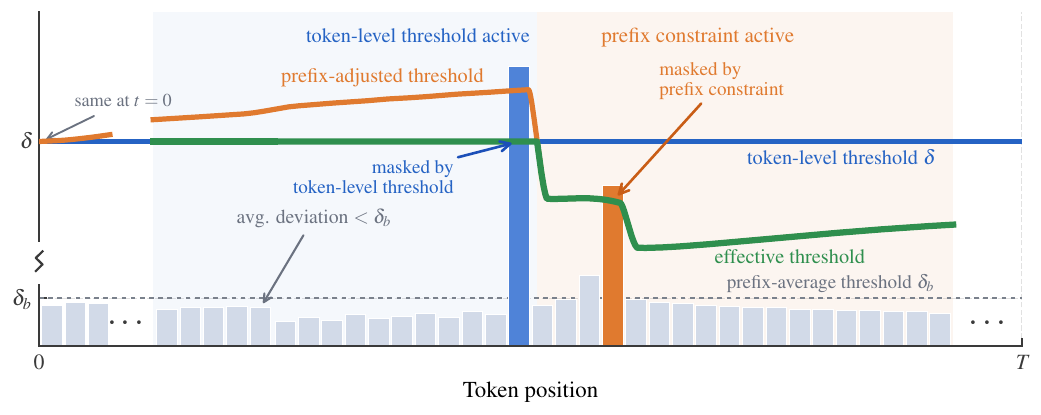}
\caption{Cumulative prefix constraint on a token window. Grey bars
are simulated weighted deviations $w_tD_t$. The blue shaded region
denotes tokens where the token-level threshold $\delta$ is active;
the orange shaded region denotes tokens where the prefix constraint
determines the effective threshold. The blue bar is
masked by the token-level threshold; the orange bar satisfies the
token-level threshold but violates the effective threshold.}
\label{fig:prefix-constraint}
\end{figure}

\subsection{Token-level mask and surrogate objective}
\label{sec:mask}

The feasibility test of Equation~\eqref{eq:pact-feasible} collects the
position-dependent token-level threshold and the cumulative prefix
constraint into a single per-token condition $I_t$. An update term is kept
whenever it moves the sampled-token ratio toward one, and otherwise only
when $I_t$ holds:
\begin{equation}
M_t^{\mathrm{\CPPOSHORT{}}}
=\ind\!\big[
\hat A_t(\rho_t-1)\le 0
\;\;\vee\;\;
I_t
\big],
\qquad
I_t:\ w_tD_t\le\delta
\;\;\wedge\;\;
S_t\le\delta+\Db W_{t-1}.
\label{eq:hard-mask}
\end{equation}
The first clause preserves update terms that move $\pi$ closer to $\mu$.
The indicator $I_t$ exclusively restricts terms that drive $\pi$ further from
$\mu$, simultaneously enforcing the token-level and prefix constraints. Substituting
this mask into the token-level ratio-advantage objective gives the \CPPO{}
surrogate
\begin{equation}
\mathcal L_\mu^{\mathrm{\CPPOSHORT{}}}(\pi)
=\E_\mu\!\left[
\sum_{t=1}^T
M_t^{\mathrm{\CPPOSHORT{}}}\,\rho_t\,\hat A_t
\right],
\label{eq:pact-loss}
\end{equation}
with the trust-region decision carried entirely by
$M_t^{\mathrm{\CPPOSHORT{}}}$ (we use \GRPO{} group-relative advantages for
$\hat A_t$ in our experiments). When the prefix has not yet drifted, $I_t$
reduces to the position-dependent token-level threshold; once accumulated
deviation drives the prefix average toward $\Db$, the prefix constraint
becomes the binding term and tightens the mask.

Algorithm~\ref{alg:pact} writes the mask for one sampled response
in the same order in which the tokens appear. The running sums
$S_t$ and $W_t$ are ordinary prefix sums: after token $t$, they
store the total weighted divergence and total weight up to that
token. The threshold at token $t$ uses $S_{t-1}$ and $W_{t-1}$,
so the current decision depends only on the preceding prefix and
the current weighted divergence $Z_t=w_tD_t$. A batched tensor
implementation computes the same quantities with a cumulative sum
over the response dimension and a one-token right shift; the loop
below is written for readability rather than for kernel efficiency.
The soft-gate variant is deferred to Appendix~\ref{app:soft}.

\begin{algorithm}[!htbp]
\caption{Token-level \CPPO{} mask for one response}
\label{alg:pact}
\begin{algorithmic}[1]
\REQUIRE \small ratios $\rho_{1:T}$, advantages $\hat A_{1:T}$,
token-level divergences $D_{1:T}$, thresholds $\delta,\Db$, weight floor
$\wmin$. \normalsize
\ENSURE token mask $M_{1:T}$.
\STATE Initialize $S_0\leftarrow0$ and $W_0\leftarrow0$.
\FOR{$t=1,\dots,T$}
\STATE Set
$w_t\leftarrow1-\frac{1-\wmin}{T-1}(t-1)$ and
$Z_t\leftarrow w_tD_t$.
\STATE Compute the effective threshold
$c_t\leftarrow\min\{\delta,\delta+\Db W_{t-1}-S_{t-1}\}$.
\IF{$\hat A_t(\rho_t-1)\le0$ or $Z_t\le c_t$}
\STATE $M_t\leftarrow1$.
\ELSE
\STATE $M_t\leftarrow0$.
\ENDIF
\STATE Update prefix sums
$S_t\leftarrow S_{t-1}+Z_t$ and $W_t\leftarrow W_{t-1}+w_t$.
\ENDFOR
\RETURN $M_{1:T}$.
\end{algorithmic}
\end{algorithm}

\FloatBarrier

\section{Experiments}
\label{sec:experiments}

The experiments test whether prefix-budgeted masking improves
reasoning RL under matched data, rollout lengths, and validation
selection windows. The baselines include ratio-based masks,
prefix-ratio objectives, divergence-based token
masks, and sequence-level trust-region masks. We first report
results across four Qwen3 settings, then use ablations to test
which \CPPO{} constraints drive the gain.

\subsection{Setup}
\label{sec:exp-setup}

All runs train on \textbf{DAPO-Math-17k}~\citep{yu2026dapo},
a set of roughly 17k verifiable mathematical reasoning prompts,
using the \texttt{verl}~\citep{10.1145/3689031.3696075} GRPO/DAPO
rollout stack. We evaluate four Qwen3~\citep{DBLP:journals/corr/abs-2505-09388}
model settings:
\textbf{Qwen3-1.7B-Base}, \textbf{Qwen3-1.7B}
(post-trained), \textbf{Qwen3-8B-Base}, and
\textbf{Qwen3-30B-A3B-Base}. The two 1.7B settings and the 8B-Base run
use $T_{\max}=8\mathrm{k}$ with $n=8$ rollouts. The
30B-A3B-Base run uses $16\mathrm{k}$ with $n=16$ rollouts.
Validation uses $\mathrm{Avg}@16$ on AIME24, AIME25, and
AIME26, and the main score is their
unweighted mean,
\textbf{AIME24/25/26 $\mathrm{Avg}@16$}. For each method, we
report the best validation score within the matched evaluation
horizon $[0,T^{\mathrm{stop}}]$ for that model, so no method is
selected after receiving additional training budget. Full
training details, per-model update budgets, and per-benchmark
breakdowns are in Appendix~\ref{app:exp-details}.

The baselines are grouped by the trust-region rule they use.
\GRPO{}~\citep{DBLP:journals/corr/abs-2402-03300} and
\CISPO{}~\citep{DBLP:journals/corr/abs-2506-13585} operate on sampled-token
ratios, with \CISPO{} using asymmetric clip thresholds.
\MinPro{}~\citep{lei2026step} adds a prefix-ratio surrogate
without a cumulative distributional budget.
\TRM{}-Max and \TRM{}-Avg~\citep{li2025trust} summarize each completed response with
a sequence-level KL criterion. The \GRPO{} row uses the Clip-Higher setting
$(\epsilon_{\mathrm{low}},\epsilon_{\mathrm{high}})=(0.2,0.28)$.
The other baseline hyperparameters follow their original papers.
The closest baseline is
\DPPO{}~\citep{qi2026rethinking}. \DPPO{} and \CPPO{} use a matched
Top-$K$ ($K{=}20$) reduced-TV score and per-model token-level
threshold scale. \DPPO{} applies this scale as a uniform token-level
threshold, while \CPPO{} uses it inside the weighted and prefix
constraints.

\CPPO{} adds two hyperparameters on top of this shared threshold scale,
both defined in Section~\ref{sec:method}. The token-level threshold scale
$\delta$ is the per-token divergence allowed at the start of a response,
exactly as in \DPPO{}. The prefix-average threshold $\Db$ is the largest
weighted average of the token-level divergences permitted over any prefix
of the response (the $S_t\le\Db W_t$ part of
Equation~\eqref{eq:pact-feasible}); it controls how much accumulated
deviation a prefix may carry, not a per-token quantity. The weight floor
$\wmin$ sets the smallest position weight, so the linear schedule
$w_t\in[\wmin,1]$ runs from a full-strength constraint on the first token
to its loosest setting on the last; a smaller $\wmin$ relaxes late tokens
more, matching the shorter remaining horizon there. We set the threshold
scale to $\delta=0.15$ for the dense models and $\delta=0.2$ for the
30B-A3B MoE model. \CPPO{} uses $\wmin=0.8$ throughout. For the post-trained reasoning
model, the average token-level divergence remains stable, so we use a
fixed $\Db=0.015$. In contrast, during the initial exploration phase
of Base-model training, the average token-level divergence is
exceptionally large before rapidly decaying as the policy stabilizes.
To avoid excessively clipping these exploratory tokens before
training stabilizes, we set $\Db$ adaptively based on each
sequence's divergence statistics on the Base models. Specifically,
for each generated sequence, we compute the top-$10\%$ quantile (the
90th percentile) of its token-level divergences as the raw $\Db$, and
then clamp this value between a minimum threshold $\Db^{\min}$
(listed in Table~\ref{tab:exp-config}) and an upper bound of
$2\Db^{\min}$.
Appendix Figure~\ref{fig:base-deltab-diagnostics} reports the realized effective $\Db$ values and the fraction of masked tokens rejected by the prefix-budget condition.

\FloatBarrier

\subsection{Main results}
\label{sec:exp-main}

Table~\ref{tab:main} reports the best validation
AIME24/25/26 $\mathrm{Avg}@16$ reached within the same
$[0,T^{\mathrm{stop}}]$ window for each model setting. This
selection rule prevents a method from benefiting from a longer
training budget on the same model.

\begin{table}[!htbp]
\centering
\caption{\textbf{Best validation AIME24/25/26 $\mathrm{Avg}@16$
across the four models} (\%, higher is better). For each
method we report the highest AIME24/25/26 $\mathrm{Avg}@16$
within the matched training-step window
$[0,T^{\mathrm{stop}}]$. The best score in each column is in
\textbf{bold} and the second-best is \underline{underlined}.
Per-benchmark breakdowns and
$T^{\mathrm{stop}}$ are in Table~\ref{tab:main-full} of
Appendix~\ref{app:exp-details}.}
\label{tab:main}
\small
\resulttablespacing
\begin{tabularx}{\textwidth}{l *{4}{C}}
\toprule
\textbf{Method} & \textbf{1.7B} & \textbf{1.7B-Base} & \textbf{8B-Base} & \textbf{30B-A3B-Base}\\
\midrule
\GRPO{}         & 27.91             & \phantom{0}8.89             & 23.96             & 38.19             \\
\MinPro{}       & 27.71             & 11.04                       & \tsecond{29.72}   & 48.12             \\
\CISPO{}        & \tsecond{28.82}   & \tsecond{11.87}             & 29.58             & \textit{collapse}  \\
\DPPO{}         & 28.19             & 10.90                       & 28.89             & \tsecond{49.23}   \\
\TRM{}-Max      & 25.21             & \phantom{0}9.72             & 26.73             & 20.27             \\
\TRM{}-Avg      & 26.87             & 11.70                       & 27.98             & 48.96             \\
\midrule
\rowcolor{ourshl}
\CPPO{} (ours)  & \tbest{31.88}     & \tbest{12.78}               & \tbest{31.11}     & \tbest{54.79}     \\
\bottomrule
\end{tabularx}
\end{table}

\CPPO{} consistently outperform all the baselines across all settings by a significant margin. 
Specifically, \CPPO{} attains the best AIME24/25/26 $\mathrm{Avg}@16$ in all settings, reaching 31.88, 12.78, 31.11, and 54.79 on Qwen3-1.7B,
Qwen3-1.7B-Base, Qwen3-8B-Base, and Qwen3-30B-A3B-Base, respectively. The margins over
the second-best method are 3.06, 0.91, 1.39, and 5.56 absolute points.
The performance of the baselines varies across models. \CISPO{} attains the
second highest validation performance on the 1.7B models, while \MinPro{} and
\DPPO{} rank second on 8B-Base and 30B-A3B-Base respectively. 

The comparison with \DPPO{} is strictly controlled, as the
two methods share the same Top-$K$ reduced-TV score and the same
per-model threshold scale $\delta$ and differ only in the weighted and
prefix constraints. Under this matched setup \CPPO{} improves on
\DPPO{} by 3.69, 1.88, 2.22, and 5.56 points across the four models.
The gain is thus attributable to how the token-level divergence is
allocated along the response, and not to a different divergence measure
or a looser threshold scale.

The largest improvement, 5.56 points, occurs on Qwen3-30B-A3B-Base, the
largest model and the only run with a $16\mathrm{k}$ rollout horizon.
This is the regime where the remaining-horizon amplification of
Section~\ref{sec:residual-bound} is most pronounced, since an early-token
deviation there propagates through a longer suffix. The same setting
separates stable from unstable trust-region rules: \CISPO{} collapses
partway through training and is omitted from this column, and the
sequence-level \TRM{}-Max degrades to 20.27, whereas \CPPO{} trains
stably to the selected checkpoint. The per-benchmark AIME24/25/26
components at the selected checkpoints are reported in
Table~\ref{tab:main-full}.

Figure~\ref{fig:main-curves} shows the validation
AIME24/25/26 $\mathrm{Avg}@16$ trajectories for the three Base-model
runs. \CPPO{} consistently maintains a performance advantage over the
baselines throughout training, confirming that the prefix-constrained trust
region yields stable and sustained performance gains. The
separation from \DPPO{} widens as training proceeds, consistent with the
prefix constraint tightening only after policy deviation has accumulated
over a prefix.

\begin{figure}[!htbp]
\tightfigtop
\centering
\includegraphics[width=\linewidth]{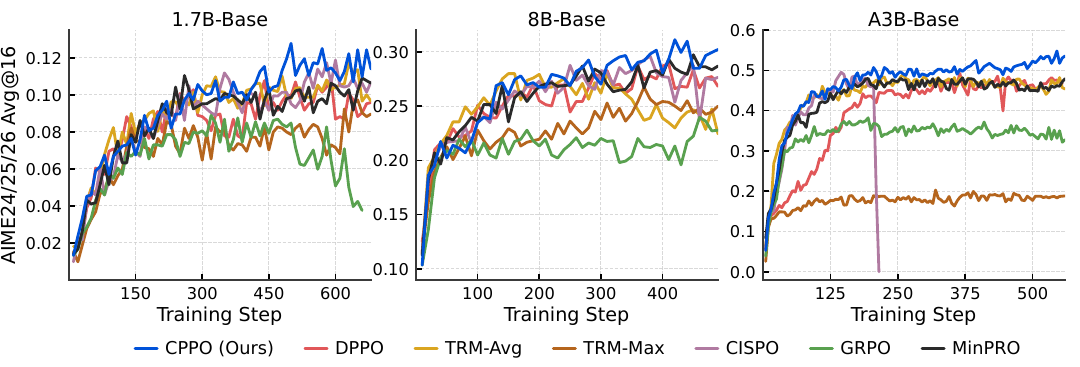}
\caption{Validation AIME24/25/26 $\mathrm{Avg}@16$ curves for
the three Base-model runs.}
\label{fig:main-curves}
\end{figure}

\subsection{Ablations}
\label{sec:exp-ablations}

We isolate the contributions of \CPPO{}'s components by incrementally
modifying the matched \DPPO{} baseline on Qwen3-1.7B and Qwen3-1.7B-Base.

\noindent\textbf{Single mechanism ablation.} The left panel of Figure~\ref{fig:ablation-main} separates the
two constraints inside \CPPO{} on Qwen3-1.7B against the matched \DPPO{}
baseline, which uses a uniform token-level threshold with no prefix
constraint. \CPPO{} w/o Position Weight adds the prefix constraint with
$w_t\equiv 1$, and \CPPO{} w/o Prefix Budget uses $w_tD_t\le\delta$ without
the prefix constraint. Both single-constraint variants outperform
\DPPO{}, while the complete \CPPO{} mask achieves the highest validation
scores. This indicates that the position weights and the cumulative
prefix budget provide independent and complementary performance gains.

\noindent\textbf{Position-weight ordering.} The middle panel evaluates whether the performance improvements of the
position-dependent threshold stem from the autoregressive order itself or
from threshold heterogeneity. The shuffled variant keeps the same multiset of
position-dependent token-level thresholds implied by $\{w_1,\dots,w_T\}$
within each rollout but randomly reassigns them to token positions.
The ordered schedule outperforms the shuffled variant, confirming that the
autoregressive position order, rather than threshold heterogeneity, drives the
performance gain.

\noindent\textbf{Mask vs. soft gate.} The right panel compares the full \CPPO{} hard mask with a
soft variant, which attenuates a token's gradient near the constraint
boundary with a non-increasing gate instead of dropping it, following
\SAPO{}~\citep{gao2025soft}; the gate construction is in
Appendix~\ref{app:soft}. The soft variant stays in the same range as the
hard mask, so we keep the hard mask as the default and treat soft gating as
an implementation choice.

\begin{figure}[!htbp]
\tightfigtop
\centering
\includegraphics[width=\linewidth]{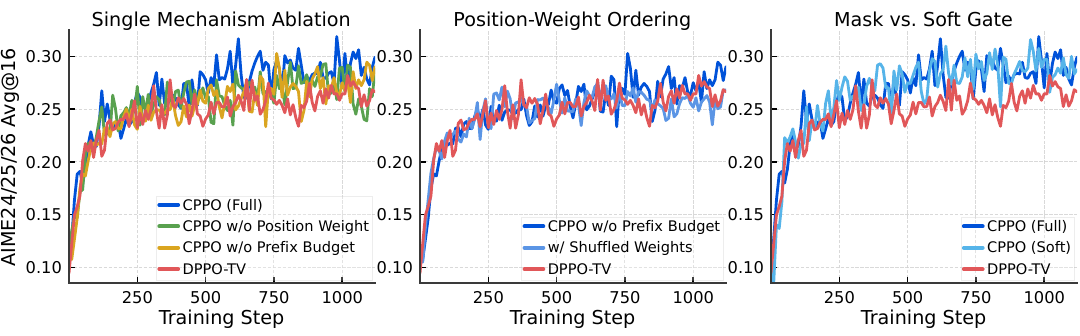}
\caption{Ablations on Qwen3-1.7B.}
\label{fig:ablation-main}
\end{figure}

\noindent\textbf{Hyperparameter sensitivity.} The left panel of Figure~\ref{fig:ablation-base} sweeps the
minimum prefix-average threshold and the weight floor around the default
$(\Db,\wmin)=(0.02,0.8)$ on Qwen3-1.7B-Base. The neighboring settings
stay close to the default and above \DPPO{}, so the gain is not tied to
a single operating point.

\noindent\textbf{KL vs. TV divergence.} The middle panel replaces the Top-$K$ reduced-TV score with a KL score
while keeping the \CPPO{} prefix constraints fixed, using the per-token
and prefix-average thresholds of \TRM{} ($\delta=0.1$ and $\Db=0.002$).
The KL configuration yields performance comparable to the TV configuration
and consistently outperforms \DPPO{}. This demonstrates that the
improvements of \CPPO{} are robust to the choice of divergence metric. As a control, the TRM Max\&Avg curve applies these same two
thresholds through \TRM{}'s sequence-level masks, which keep or drop an
entire response by its maximum and mean token-level KL rather than by a
cumulative prefix budget. This variant matches \DPPO{} and stays below
\CPPO{}, indicating that the gain comes from enforcing the thresholds as
a prefix budget rather than from the threshold values.

\noindent\textbf{Binary vs. Top-$K$ approximation.} The right panel replaces the Top-$K$ reduced-TV score with the simpler
Binary-TV partition used by \DPPO{}. Both configurations maintain similar
validation scores and consistently outperform \DPPO{}. This indicates the
performance improvement is robust to the choice of divergence metric and its
approximation granularity, aligning with the \DPPO{} ablation of
\citet{qi2026rethinking} that finds Binary and Top-$K$ estimators yield
comparable results. The prefix budget, not the
divergence estimator, drives the improvement.

\begin{figure}[!htbp]
\tightfigtop
\centering
\includegraphics[width=\linewidth]{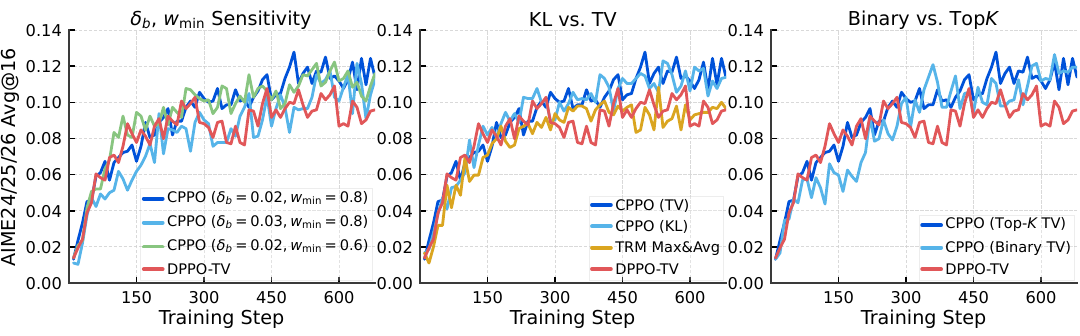}
\caption{Qwen3-1.7B-Base ablations.}
\label{fig:ablation-base}
\end{figure}
\section{Conclusion}
\label{sec:discussion}

This work revisits the token-level trust region used in reasoning RL. A
uniform token-level threshold applies an identical constraint across all
positions. This ignores the autoregressive nature of generation, where
early deviations affect longer suffixes and per-token errors accumulate
within the conditioning prefix. Starting from the finite-horizon performance difference identity,
we derive a prefix-constrained policy-improvement bound that makes both
effects explicit and turn it into \CPPO{}, a token-level mask with two
mechanisms: a decreasing position weight that makes the threshold tighter at
early positions, and a cumulative prefix budget that reduces the divergence
allowed at a later update once preceding deviations have driven the weighted
prefix average up.

Both mechanisms act through the masking decision alone, so \CPPO{} reuses
the \PPO{}/\GRPO{} ratio-advantage objective and the same per-token
divergence as \DPPO{}, introducing no additional loss terms. Across four Qwen3 settings
spanning dense and MoE models and Base and post-trained checkpoints, \CPPO{}
attains the best validation AIME24/25/26 $\mathrm{Avg}@16$, and the
ablations attribute the gain to the position weight and the prefix budget
rather than to the divergence estimator.

\section*{Acknowledgement}
\label{sec:acknowledgement}

We are grateful to all the members of the Hunyuan Multimodal Foundational RL Team for their support. We especially thank Tianyu Pang for his invaluable advice and feedback on this work.

\clearpage
\bibliography{references}

\clearpage
\appendix
\section{Related work}
\label{app:related}

\CPPO{} is a trust-region method for RLVR. We organize related work
by the statistic each method uses to constrain policy movement;
\CPPO{}'s contribution is prefix-budgeted, position-aware token
masking under a fixed token-level divergence statistic, not a new way
to measure that divergence.

\noindent\textbf{Sampled-ratio methods.}
\PPO{}-style methods~\citep{schulman2017proximal} restrict updates
through the sampled-token importance ratio $\rho_t$.
\GRPO{}~\citep{DBLP:journals/corr/abs-2402-03300} adapts this template
to LLMs with group-relative advantages from a verifier;
Dr.GRPO~\citep{liu2025understanding} and
REINFORCE++~\citep{hu2025reinforce++} change the advantage
normalization or the policy-gradient estimator while keeping the
sampled-ratio update.
DAPO~\citep{yu2026dapo} relaxes the upper clip range to preserve
exploratory high-ratio updates;
\CISPO{}~\citep{DBLP:journals/corr/abs-2506-13585} clips the
importance-sampling weight rather than the token update itself,
preserving gradients from low-probability tokens that token-level
clipping would otherwise suppress.
GSPO~\citep{DBLP:journals/corr/abs-2507-18071} extends the importance
ratio from tokens to sequences with sequence-level clipping;
GMPO~\citep{zhao2026geometricmean} replaces the arithmetic mean of
token ratios in the surrogate with a geometric mean, reducing the
influence of outlier tokens.
\MinPro{}~\citep{lei2026step} substitutes a non-cumulative
minimum-prefix-ratio surrogate for the unstable cumulative prefix
ratio.
GPG~\citep{DBLP:journals/corr/abs-2504-02546} drops the surrogate and
the ratio-clip test altogether for a direct policy-gradient
estimator, outside the trust-region family.
Apart from GPG, these methods control the update through one or more
sampled importance ratios, which only provides a single-sample estimate
of the policy shift.

\noindent\textbf{Distributional-divergence trust regions.}
A second line of work replaces the sampled-ratio test with a
token-level distributional divergence $D_t$.
\DPPO{}~\citep{qi2026rethinking} uses TV or KL as $D_t$ via Binary or
Top-$K$ reduced approximations and applies a uniform threshold across
token positions.
\TRM{}~\citep{li2025trust} applies a sequence-level mask: if the
maximum token divergence (\TRM{}-Max) or the response-mean divergence
(\TRM{}-Avg) exceeds a threshold, the entire response is excluded
from the update.
Classical trust-region
methods~\citep{kakade2002approximately,schulman2015trust,achiam2017constrained,peters2010relative,DBLP:journals/corr/abs-1806-06920,DBLP:journals/corr/abs-1909-12238}
constrain or regularize KL in classical RL settings, providing the
trust-region/relative-entropy lineage rather than LLM token-level
divergence-budget methods.
Trust-region-guided PPO variants further modify the clip threshold or
rollback behavior using divergence
information~\citep{DBLP:journals/corr/abs-1901-10314,DBLP:conf/uai/WangHT19};
separate empirical analyses document how PPO/TRPO implementation
choices and clip thresholds affect
performance~\citep{Engstrom2020Implementation,andrychowicz2021what}.
These methods change how the trust region measures policy movement;
\CPPO{} addresses a separate question of where that movement is
allowed to accumulate. Given a fixed token-level divergence, \CPPO{}
weights early positions more strongly through $w_t$ and caps the
weighted prefix average through $\Db$, rather than introducing a new
measurement of token-level shift.
Their feasibility constraints map onto specializations of \CPPO{}:
\TRM{}-Max to the uniform-threshold sequence-level specialization
without a prefix budget ($w_t\equiv 1$, no $\Db$), and \TRM{}-Avg to
a full-response-average specialization of \CPPO{}'s prefix
constraints that is weaker by up to a factor $2-2/T$ at
$w_t\equiv 1$. Algorithmically, however, \TRM{} masks whole
sequences while \CPPO{} masks individual token updates.
In our experiments, \DPPO{} and \CPPO{} use a matched Top-$K$
reduced-TV statistic and per-model threshold scale, so the empirical
comparison isolates the prefix constraints from the divergence
statistic.

\noindent\textbf{Complementary directions.}
Several recent methods modify the policy objective in orthogonal
dimensions and are complementary to \CPPO{}.
NFPO~\citep{yoon2026multi} and FIPO~\citep{ma2026fipo} change the
weight attached to each token update through future-dependent
likelihood-ratio or KL signals, whereas \CPPO{} keeps the token
surrogate and changes how policy deviation may accumulate across
prefixes.
Entropy-based analyses identify high-entropy minority tokens that
act as reasoning
forks~\citep{DBLP:journals/corr/abs-2506-01939} and study how policy
entropy evolves to prevent collapse during RLVR
training~\citep{cui2025the}; in contrast, \CPPO{} addresses the distinct
issue of how the allowed deviation should be allocated across positions and
prefixes.
Soft-masking methods~\citep{gao2025soft} replace hard violation
indicators with non-increasing attenuation, and \CPPO{}'s violation
score $x_t^{\CPPOSHORT{}}$ admits the same soft variant
(Appendix~\ref{app:soft}).

\begin{table}[!ht]
\centering
\caption{Trust-region methods grouped by the statistic they use and
where the constraint is applied. Here token-level TV/KL denotes the
divergence between rollout and target next-token distributions at a
sampled prefix. In the matched comparison, \DPPO{} and \CPPO{} use a
matched Top-$K$ reduced-TV statistic; prefix-ratio objectives such as
\MinPro{} are not divergence-budget methods.}
\label{tab:method-distinction}
\small
\setlength{\tabcolsep}{5pt}
\begin{tabular}{@{}>{\raggedright\arraybackslash}p{0.16\textwidth}
                  >{\raggedright\arraybackslash}p{0.31\textwidth}
                  >{\raggedright\arraybackslash}p{0.46\textwidth}@{}}
\toprule
Method class & Statistic used by the rule & Constraint applied \\
\midrule
\PPO/\GRPO
& sampled-token ratio $\rho_t$
& clips sampled ratios in the surrogate;\newline
no distributional divergence budget \\
\DPPO{}
& token-level divergence $D_t$
& uses one divergence threshold;\newline
shared across token positions \\
\TRM-Max
& response maximum $\max_t D_t$
& constrains the largest token divergence;\newline
computed over the full response \\
\TRM-Avg
& response average $T^{-1}\sum_t D_t$
& constrains the average divergence;\newline
computed over the full response \\
\CPPO{} (ours)
& token-level divergence $D_t$\newline
with prefix sums
& token-level threshold + prefix-average constraints;\newline
position weights $w_t$ vary the effective threshold \\
\bottomrule
\end{tabular}
\end{table}

\section{Full Proofs and Theoretical Details}
\label{app:fulltheory}

This appendix provides the full proofs of the results stated
in Section~\ref{sec:theory}, together with the corollaries
that characterize \CPPO{}'s relationship to position-agnostic
and sequence-level methods, and the product-form suffix bound.

\subsection{Finite-horizon performance difference identity}

The following lemma is the exact identity stated as
Equation~\eqref{eq:exact-id} in Section~\ref{sec:identity}.

\begin{lemma}[Performance difference identity]
\label{lem:exact-id}
Under common support, let
$\rho_{a:b}:=\prod_{j=a}^b\rho_j$ and
$\rho_{T+1:T}:=1$. Then
\begin{equation}
\begin{aligned}
J(\pi)-J(\mu)
&=L'_\mu(\pi)-\Delta(\mu,\pi),\\
L'_\mu(\pi)
&:=\E_\mu\!\left[
R(x,y)\sum_{t=1}^T(\rho_t-1)
\right],\\
\Delta(\mu,\pi)
&:=\E_\mu\!\left[
R(x,y)\sum_{t=1}^T(\rho_t-1)(1-\rho_{t+1:T})
\right].
\end{aligned}
\end{equation}
Hence $J(\pi)-J(\mu)\ge L'_\mu(\pi)-|\Delta(\mu,\pi)|$.
\end{lemma}

\begin{proof}
By importance sampling,
\begin{equation*}
J(\pi)-J(\mu)
=
\E_\mu\!\left[
R(x,y)(\rho_{1:T}-1)
\right].
\end{equation*}
The reverse telescoping identity
\begin{equation*}
\rho_{1:T}-1
=
\sum_{t=1}^T(\rho_t-1)\rho_{t+1:T}
\end{equation*}
then gives
\begin{align*}
J(\pi)-J(\mu)
&=
\E_\mu\!\left[
R(x,y)\sum_{t=1}^T(\rho_t-1)\rho_{t+1:T}
\right]\\
&=
\E_\mu\!\left[
R(x,y)\sum_{t=1}^T(\rho_t-1)
\right]
-
\E_\mu\!\left[
R(x,y)\sum_{t=1}^T(\rho_t-1)(1-\rho_{t+1:T})
\right].
\end{align*}
The two terms on the last line are exactly
$L'_\mu(\pi)$ and $\Delta(\mu,\pi)$. The lower bound follows from
$-\Delta(\mu,\pi)\ge-|\Delta(\mu,\pi)|$.
\end{proof}

\subsection{Suffix TV via maximal coupling}

\begin{lemma}[Suffix TV bound]
\label{lem:suffix-tv}
If $D_j\le\ell_j$ pathwise for all $j>t$, then for any $s_{t+1}$,
\begin{equation}
\DTV(P_\mu^{t+1:T}(\cdot\mid s_{t+1}),P_\pi^{t+1:T}(\cdot\mid s_{t+1}))
\le 1-\prod_{j=t+1}^T(1-\ell_j)\le \sum_{j=t+1}^T\ell_j.
\end{equation}
\end{lemma}

\begin{proof}
Construct a stepwise maximal coupling of the two suffix processes.
As long as the two sampled suffixes have not disagreed, they share the
same state $s_j$. The one-step disagreement probability at that state is
\[
\DTV\!\left(\mu(\cdot\mid s_j),\pi(\cdot\mid s_j)\right)
=D_j
\le \ell_j .
\]
Therefore the conditional probability of no disagreement over the whole
suffix is at least
\[
\prod_{j=t+1}^T(1-\ell_j).
\]
The coupling characterization of total variation gives
\begin{align*}
\DTV\!\left(
P_\mu^{t+1:T}(\cdot\mid s_{t+1}),
P_\pi^{t+1:T}(\cdot\mid s_{t+1})
\right)
&\le
\Pr\!\left(Y_{t+1:T}^{\mu}\ne Y_{t+1:T}^{\pi}\mid s_{t+1}\right)\\
&\le
1-\prod_{j=t+1}^T(1-\ell_j).
\end{align*}
Since total variation is at most one, we may take
$\ell_j\in[0,1]$. The union bound then yields
\[
1-\prod_{j=t+1}^T(1-\ell_j)
\le
\sum_{j=t+1}^T\ell_j .
\]
\end{proof}

\subsection{Remaining-horizon bound on the surrogate residual}

\begin{proposition}[Surrogate-residual bound under token-level divergence thresholds]
A maximal-coupling argument on the suffix likelihood ratio gives the following bound on
the surrogate residual. This is the first inequality of
Equation~\eqref{eq:residual-bound} in Section~\ref{sec:residual-bound}.
\label{prop:residual-bound}
If $D_t\le\ell_t$ pathwise for all $t$, then
\begin{equation}
|\Delta(\mu,\pi)|\;\le\;
4\xi\sum_{t=1}^{T-1}u_t\,\sum_{j=t+1}^T\ell_j.
\end{equation}
\end{proposition}

\begin{proof}
Lemma~\ref{lem:exact-id} and $|R(x,y)|\le\xi$ give
\begin{equation*}
|\Delta(\mu,\pi)|
\le
\xi\sum_{t=1}^T
\E_\mu\!\left[
|\rho_t-1|\,|1-\rho_{t+1:T}|
\right].
\end{equation*}
Fix $t<T$ and condition on $s_{t+1}$. The future likelihood ratio
$\rho_{t+1:T}$ is the Radon--Nikodym derivative of the suffix law
under $\pi$ with respect to the suffix law under $\mu$, so
\begin{align*}
\E_\mu\!\left[|1-\rho_{t+1:T}|\mid s_{t+1}\right]
&=
\sum_{y_{t+1:T}}
P_\mu^{t+1:T}(y_{t+1:T}\mid s_{t+1})
\left|
1-
\frac{P_\pi^{t+1:T}(y_{t+1:T}\mid s_{t+1})}
{P_\mu^{t+1:T}(y_{t+1:T}\mid s_{t+1})}
\right|\\
&=
2\DTV\!\left(
P_\mu^{t+1:T}(\cdot\mid s_{t+1}),
P_\pi^{t+1:T}(\cdot\mid s_{t+1})
\right)\\
&\le
2\sum_{j=t+1}^T\ell_j ,
\end{align*}
where the last line uses Lemma~\ref{lem:suffix-tv}. The sampled-token
ratio contributes
\begin{align*}
\E_{y_t\sim\mu(\cdot\mid s_t)}
\!\left[|\rho_t-1|\mid s_t\right]
&=
\sum_{y_t}\mu(y_t\mid s_t)
\left|
\frac{\pi(y_t\mid s_t)}{\mu(y_t\mid s_t)}-1
\right|\\
&=
2D_t .
\end{align*}
Taking expectation over $s_t\sim d_t^\mu$ gives
$\E_\mu|\rho_t-1|=2u_t$. Combining the two displayed bounds,
\begin{equation*}
\E_\mu\!\left[
|\rho_t-1|\,|1-\rho_{t+1:T}|
\right]
\le
4u_t\sum_{j=t+1}^T\ell_j .
\end{equation*}
The term $t=T$ is zero because $\rho_{T+1:T}=1$. Summing over
$t=1,\dots,T-1$ proves the claim.
\end{proof}

\subsection{\CPPO{} policy-improvement bound (Theorem~\ref{thm:pact-main})}

\begin{proof}[Proof of Theorem~\ref{thm:pact-main}]
We prove the theorem in three steps.

\emph{Step 1: reduce the residual to a weighted sum of token-level
divergences.}
The constraint $w_tD_t\le c_t$ in~\eqref{eq:prefix-feasible-main} implies
\[
D_t\le \ell_t:=\frac{c_t}{w_t},
\qquad
\bar\ell:=\max_t\ell_t .
\]
By Proposition~\ref{prop:residual-bound},
\begin{equation*}
|\Delta(\mu,\pi)|
\le
4\xi\sum_{t=1}^{T-1}u_t\sum_{j=t+1}^T\ell_j
\le
\sum_{t=1}^{T-1}\lambda_tu_t,
\qquad
\lambda_t:=4\xi\bar\ell(T-t).
\end{equation*}

\emph{Step 2: convert the prefix constraints into prefix inequalities in
expectation.}
Taking expectations in the prefix-budget constraints gives, for every
$m=1,\dots,T-1$,
\[
\sum_{j=1}^m w_ju_j
\le
\Db\sum_{j=1}^m w_j .
\]
Define the centered prefix slack
\[
S_m:=\sum_{j=1}^m w_j(u_j-\Db),
\qquad
S_0:=0 .
\]
Then $S_m\le0$ for every $m<T$. Also define
\[
\Delta S_t:=S_t-S_{t-1}=w_t(u_t-\Db),
\qquad
r_t:=\frac{\lambda_t}{w_t}.
\]

\emph{Step 3: apply Abel summation.}
Using the previous definitions,
\begin{align*}
\sum_{t=1}^{T-1}\lambda_tu_t
-\Db\sum_{t=1}^{T-1}\lambda_t
&=
\sum_{t=1}^{T-1}\lambda_t(u_t-\Db)\\
&=
\sum_{t=1}^{T-1}r_t\,\Delta S_t\\
&=
r_{T-1}S_{T-1}
+\sum_{t=1}^{T-2}(r_t-r_{t+1})S_t .
\end{align*}
By assumption, $r_t$ is non-increasing, so
$r_t-r_{t+1}\ge0$; also $r_{T-1}\ge0$ and $S_t\le0$. Hence the
right-hand side is non-positive, and therefore
\begin{equation*}
\sum_{t=1}^{T-1}\lambda_tu_t
\le
\Db\sum_{t=1}^{T-1}\lambda_t .
\end{equation*}
Finally,
\[
\sum_{t=1}^{T-1}\lambda_t
=
4\xi\bar\ell\sum_{t=1}^{T-1}(T-t)
=
2\xi T(T-1)\bar\ell .
\]
Combining this residual bound with Lemma~\ref{lem:exact-id} gives
\[
J(\pi)-J(\mu)
\ge
L'_\mu(\pi)-2\xi T(T-1)\bar\ell\Db .
\]
If $c_t\equiv\delta$ and $w_t\in[\wmin,1]$, then
\[
\bar\ell
=
\max_t\frac{\delta}{w_t}
\le
\frac{\delta}{\wmin}.
\]
Substituting this into the bound gives
Equation~\eqref{eq:pact-main-impl}.
\end{proof}

\subsection{Linear schedule satisfies the monotonicity}
\label{app:prop-linear-mono}

\begin{proposition}[Linear schedule monotonicity]
\label{prop:linear-mono}
For $w_t=1-\tfrac{1-\wmin}{T-1}(t-1)$ and $g_t:=(T-t)/w_t$,
$g_t-g_{t+1}=\wmin/(w_tw_{t+1})>0$, so $r_t=4\xi\bar\ell\,g_t$ is
strictly decreasing on $t=1,\dots,T-1$. Therefore
Theorem~\ref{thm:pact-main} applies to the implemented schedule.
\end{proposition}

\begin{proof}
Write
\[
c:=\frac{1-\wmin}{T-1},
\qquad
w_t=1-c(t-1),
\qquad
g_t=\frac{T-t}{w_t}.
\]
Then
\begin{align*}
g_t-g_{t+1}
&=
\frac{(T-t)w_{t+1}-(T-t-1)w_t}{w_tw_{t+1}}\\
&=
\frac{w_t-c(T-t)}{w_tw_{t+1}}\\
&=
\frac{1-c(T-1)}{w_tw_{t+1}}\\
&=
\frac{\wmin}{w_tw_{t+1}}
>0 .
\end{align*}
Thus $g_t$ is strictly decreasing, and
$r_t=4\xi\bar\ell\,g_t$ is strictly decreasing as well.
\end{proof}

\subsection{Corollaries: uniform threshold and implementation threshold}
\label{app:corollaries}

These two corollaries supply the residual-constant comparisons used in
the discussion after Theorem~\ref{thm:pact-main} in
Section~\ref{sec:thm}: Corollary~\ref{cor:uniform} gives the uniform
token-level threshold case ($\Db/\delta$ ratio) and
Corollary~\ref{cor:impl} the position-dependent implementation case
($\Db/(\delta\wmin)$ ratio).

\begin{corollary}[Uniform token-level threshold]
\label{cor:uniform}
If $c_t=w_t\delta$, equivalently $\ell_t\equiv\delta$, then
$\bar\ell=\delta$ and
\begin{equation}
J(\pi)-J(\mu)
\;\ge\;
L'_\mu(\pi)-2\xi T(T-1)\delta\Db .
\end{equation}
The residual constants are
\[
C_{\CPPOSHORT{}}
:=2\xi T(T-1)\delta\Db,
\qquad
C_{\mathrm{uniform}}
:=2\xi T(T-1)\delta^2 .
\]
Thus
\[
\frac{C_{\CPPOSHORT{}}}{C_{\mathrm{uniform}}}
=
\frac{\Db}{\delta}.
\]
\textup{\CPPO{}} is tighter in the intended regime
$\Db<\delta\le 1$.
\end{corollary}

\begin{corollary}[Position-dependent token-level threshold]
\label{cor:impl}
If $w_tD_t\le\delta$, equivalently $\ell_t=\delta/w_t$, with
$w_t\in[\wmin,1]$, then
$\bar\ell\le\delta/\wmin$, hence
\begin{equation}
J(\pi)-J(\mu)
\;\ge\;
L'_\mu(\pi)-
2\xi T(T-1)\frac{\delta}{\wmin}\Db .
\end{equation}
Against the position-independent token-level divergence constant
$C_{\mathrm{uniform}}(\delta)=2\xi T(T-1)\delta^2$, the
implemented \textup{\CPPO{}} constant is
\[
C_{\mathrm{impl}}
:=
2\xi T(T-1)\frac{\delta}{\wmin}\Db,
\qquad
\frac{C_{\mathrm{impl}}}{C_{\mathrm{uniform}}(\delta)}
=
\frac{\Db}{\delta\wmin}.
\]
It is tighter whenever $\Db<\delta\,\wmin$. The linear schedule
$w_t=1-\tfrac{1-\wmin}{T-1}(t-1)$ used throughout our
experiments satisfies the monotonicity assumption of
Theorem~\ref{thm:pact-main} by
Proposition~\ref{prop:linear-mono}.
\end{corollary}

\subsection{Implementation slack from the initial prefix}
\label{app:initial-prefix-slack}

Under Equation~\eqref{eq:pact-feasible}, the first update that
moves the sampled-token ratio away from one matches the uniform
token-level divergence baseline threshold. This adds
a constant initial slack to the clean prefix inequality used in
Theorem~\ref{thm:pact-main}. The slack changes constants but not
the leading prefix-budget term.

\begin{proposition}[Initial prefix slack bound]
\label{prop:initial-prefix-slack}
Assume the weighted constraint gives
$D_t\le\bar\ell$ and, in expectation,
\[
\sum_{j=1}^m w_j u_j \le \Db W_m+\eta
\qquad (m=1,\dots,T-1)
\]
for some constant slack $\eta\ge0$. Under the same monotonicity
condition as Theorem~\ref{thm:pact-main},
\begin{equation}
|\Delta(\mu,\pi)|
\le
2\xi T(T-1)\bar\ell\Db
\;+\;
\eta\,\frac{4\xi\bar\ell(T-1)}{w_1}.
\end{equation}
For the implemented linear schedule, $w_1=1$. The initial prefix
slack in Equation~\eqref{eq:pact-feasible} has $\eta\le\delta$,
so its additional residual term is $O(T\bar\ell\delta)$.
\end{proposition}

\begin{proof}
Define
\[
A_m:=\sum_{j=1}^m w_j(u_j-\Db),
\qquad
A_0:=0 .
\]
The assumed slack gives $A_m\le\eta$ for every $m<T$. As in the
proof of Theorem~\ref{thm:pact-main}, set
\[
\lambda_t:=4\xi\bar\ell(T-t),
\qquad
r_t:=\frac{\lambda_t}{w_t}.
\]
Abel summation gives
\begin{align*}
\sum_{t=1}^{T-1}\lambda_tu_t
-\Db\sum_{t=1}^{T-1}\lambda_t
&=
r_{T-1}A_{T-1}
+\sum_{t=1}^{T-2}(r_t-r_{t+1})A_t .
\end{align*}
Because $r_t$ is non-increasing, all coefficients on the right are
nonnegative. Therefore
\begin{align*}
r_{T-1}A_{T-1}
+\sum_{t=1}^{T-2}(r_t-r_{t+1})A_t
&\le
\eta\left(
r_{T-1}+\sum_{t=1}^{T-2}(r_t-r_{t+1})
\right)\\
&=
\eta r_1\\
&=
\eta\,\frac{4\xi\bar\ell(T-1)}{w_1}.
\end{align*}
Combining this with Proposition~\ref{prop:residual-bound} proves the
displayed residual bound.

It remains to identify the slack induced by the implemented
constraint test. Equation~\eqref{eq:pact-feasible} implies
\[
S_m
\le
\delta+\Db W_{m-1}
\le
\delta+\Db W_m .
\]
Taking expectations gives the proposition's prefix-slack condition
with $\eta\le\delta$.
\end{proof}

\subsection{Sequence-level methods as special cases}
\label{app:seq-comparison}

Let $W:=\sum_{t<T}w_t$.

\noindent\textbf{Uniform token-level divergence masks recover the position-agnostic branch.}
If $w_t\equiv 1$ and the prefix budget is removed, only the token-level
threshold $D_t\le\delta$ remains; this is the position-agnostic
specialization underwriting the textbook bound, of which the
uniform-TV mask is one implementation. \TRM{}-Max corresponds to
the sequence-level form of this uniform token-level threshold:
\[
\max_t D_t\le\delta
\quad\Longrightarrow\quad
u_t\le\delta \quad \text{for every } t .
\]
It therefore recovers the same quadratic constant
\[
C_{\mathrm{max}}
=
2\xi T(T-1)\delta^2 .
\]
\CPPO{} adds the prefix conditions in~\eqref{eq:prefix-feasible-main}.
Under this uniform token-level threshold, the constant becomes
\[
C_{\CPPOSHORT{}}
=
2\xi T(T-1)\delta\Db,
\qquad
\frac{C_{\CPPOSHORT{}}}{C_{\mathrm{max}}}
=
\frac{\Db}{\delta}.
\]

\noindent\textbf{\TRM{}-Avg is a terminal-only relaxation of the prefix constraints.}
With $w_t\equiv 1$ and only the final prefix constraint kept, a
direct bound on the residual gives
\[
C_{\mathrm{terminal}}
\le
4\xi\bar\ell(T-1)^2\Db .
\]
\CPPO{}'s Abel bound using every prefix at $w\equiv 1$ is
\[
C_{\CPPOSHORT{}}^{w\equiv 1}
=
2\xi T(T-1)\bar\ell\Db .
\]
Hence
\begin{equation}
\frac{C_{\mathrm{terminal}}}{C_{\CPPOSHORT{}}^{w\equiv 1}}
\;=\;2-\frac{2}{T}.
\end{equation}
For arbitrary $w_t>0$, the weighted min-max inequality
\[
W\max_{t<T}\frac{\lambda_t}{w_t}
\ge
\sum_{t<T}\lambda_t
\]
yields
\[
C_{\mathrm{terminal}}(w)
\ge
C_{\CPPOSHORT{}}(w),
\]
with equality if and only if $w_t\propto\lambda_t\propto T-t$.
The same algebra applies to length-neutral KL or TV variants of
the sequence-level comparison.

\subsection{Product-form suffix bound}
\label{app:product}

The main paper uses the linear-suffix specialization of the
residual bound. Lemma~\ref{lem:suffix-tv} also gives the
tighter \emph{product} bound
\[
\beta_t
:=
1-\prod_{j=t+1}^T(1-\ell_j)
\le
\sum_{j=t+1}^T\ell_j .
\]
Substituting $\beta_t$ for the linear sum yields
\[
|\Delta(\mu,\pi)|
\le
4\xi\sum_{t=1}^{T-1}u_t\,\beta_t,
\]
with corresponding $\lambda_t^\beta:=4\xi\beta_t$. Under the
analogous monotonicity $r_t^\beta:=\lambda_t^\beta/w_t$
non-increasing, the same Abel-summation step gives, with
$\ell_j\equiv\delta$,
\begin{align*}
|\Delta|
&\le
\Db\sum_{t=1}^{T-1}\lambda_t^\beta\\
&=
4\xi\Db\sum_{t=1}^{T-1}\bigl[1-(1-\delta)^{T-t}\bigr]\\
&=
O(T\,\Db),
\end{align*}
i.e.\ a true $O(T)$ branch (not requiring $\Db=O(1/T)$).
The monotonicity for $\lambda_t^\beta/w_t$ is not identical
to the linear-branch monotonicity for $\lambda_t/w_t$, since
$\beta_t$ saturates at $1$ for early positions and at
$\to 0$ as $t\to T$, while $T-t$ does not. For small
$\delta$, $\beta_t\approx(T-t)\delta$ and the two
monotonicity conditions coincide. We do not use this branch
as the main quantitative claim in the paper.

\subsection{Abel summation and tightness}
\label{app:abel-discussion}

The proof of Theorem~\ref{thm:pact-main} relies on prefix
constraints rather than a pointwise estimate. The residual bound
contributes the decreasing remaining-horizon coefficients
$\lambda_t\propto T-t$, while the prefix constraints supply the
family of prefix inequalities $S_m\le 0$.
A pointwise upper bound on $\sum_{t<T}\lambda_t u_t$ ignores the
prefix structure entirely and yields nothing better than
$O(T\,\delta\,\bar\ell)$; Abel summation, the discrete analogue of
integration by parts, instead rewrites the same sum as
\[
\sum_{t<T}\lambda_t u_t
-\Db\sum_{t<T}\lambda_t
=
r_{T-1}S_{T-1}
+\sum_{t=1}^{T-2}(r_t-r_{t+1})\,S_t,
\]
which expresses the residual as a sum of prefix slacks $S_t\le 0$
weighted by the gaps $r_t-r_{t+1}$. Both factors carry the correct
sign as soon as $r_t=\lambda_t/w_t$ is non-increasing, so every
term contributes to tightening the bound. The condition
$r_t-r_{t+1}\ge 0$ is therefore the formal statement that the
constraint places larger weights where the remaining-horizon
coefficient is larger, and the bound saturates when the two
profiles are matched ($w_t\propto\lambda_t\propto T-t$).

\noindent\textbf{Sharpness.}
Among all nonnegative sequences $\{u_t\}$ satisfying the same
prefix inequalities $S_m\le 0$ for every $m$, the upper bound
\[
\sum_{t<T}\lambda_t u_t
\le
\Db\sum_{t<T}\lambda_t
\]
is attained by the saturating sequence $u_t=\Db$ whenever $r_t$ is
non-increasing. The bound is therefore tight for this class of
prefix constraints and for the residual bound in
Proposition~\ref{prop:residual-bound}: it cannot be improved without
strengthening one of them.

\subsection{Technical Lemmas}

\noindent\textbf{Total variation and $L^1$ distance.}
\[
\|P-Q\|_{\mathrm{TV}}
=
\frac12\|P-Q\|_1 .
\]

\noindent\textbf{Likelihood-ratio identity.}
\[
\E_{X\sim q}\left|1-\frac{p(X)}{q(X)}\right|
=
\int |q-p|
=
2\DTV(p,q).
\]

\noindent\textbf{Coupling characterization of total variation.}
\[
\DTV(P,Q)
=
\inf \Pr(X\ne Y),
\]
with maximal coupling attaining the infimum.

\noindent\textbf{Weighted averaging inequality.}
For $a_t,w_t>0$,
\[
\max_t\frac{a_t}{w_t}
\ge
\frac{\sum_t a_t}{\sum_t w_t},
\]
with equality if and only if $a_t/w_t$ is constant.

\noindent\textbf{Abel summation identity.}
For $\{r_t\}_{t=1}^n$ and partial sums $S_t=\sum_{j\le t}\Delta_j$,
\[
\sum_{t=1}^n r_t\Delta_t
=
r_nS_n+\sum_{t=1}^{n-1}(r_t-r_{t+1})S_t .
\]

\section{Soft-gate details}
\label{app:soft}

The main paper focuses on the hard \CPPO{} mask. Soft
trust-region gates can also attenuate gradients near the boundary
of the constraint set; this appendix collects the construction,
default choices, and the mixture-policy discussion deferred from
Section~\ref{sec:mask}. The soft variant is evaluated empirically in
the hard-vs-soft ablation of Section~\ref{sec:exp-ablations}
(right panel of Figure~\ref{fig:ablation-main}).

\subsection{Gradient-scaling interpretation}

The hard mask of Equation~\eqref{eq:hard-mask} tests the feasibility
condition $I_t$, which holds when both the token-level threshold
$w_tD_t\le\delta$ and the prefix constraint $S_t\le\delta+\Db W_{t-1}$ are
met. To attenuate gradients smoothly near this boundary rather than
dropping them, we measure how close a token is to violating $I_t$ with the
normalized violation score
\begin{equation}
x_t^{\CPPOSHORT{}}
:=\max\!\left\{
\frac{w_tD_t}{\delta},\,
\frac{S_t}{\delta+\Db W_{t-1}}
\right\},
\label{eq:soft-score}
\end{equation}
which is at most one exactly when $I_t$ holds and exceeds one in proportion
to the larger of the two constraint violations.

Multiplying the loss by a soft gate $g(x_t^{\CPPOSHORT{}})$ does not by itself
construct a softened policy; it scales each token's gradient contribution.
For a non-increasing gate $g\colon[0,\infty)\to[0,1]$ with
$g(x)=1$ for $x\le 1$ and $x g(x)\le 1$ for $x>1$, the
\emph{effective normalized violation} per token,
$x_t^{\CPPOSHORT{}}\,g(x_t^{\CPPOSHORT{}})$, is bounded by
\[
x_t^{\CPPOSHORT{}}\,g(x_t^{\CPPOSHORT{}})
\le
1 .
\]
This bounds the scaled violation that enters the gradient term. The
corresponding soft \CPPO{} mask is
\begin{equation}
M_t^{\mathrm{soft}}=
\begin{cases}
1, & \operatorname{sgn}(\hat A_t)(\rho_t-1)\le 0,\\
g(x_t^{\CPPOSHORT{}}), & \text{otherwise}.
\end{cases}
\end{equation}

\subsection{Mixture-policy construction}

To obtain a formal guarantee that mirrors Theorem~\ref{thm:pact-main},
define a mixture policy
\begin{equation}
\pi_g(\cdot\mid s)=(1-g_s)\,\mu(\cdot\mid s)+g_s\,\pi(\cdot\mid s),
\qquad g_s\in[0,1].
\end{equation}
For each state $s$,
\[
\pi_g(\cdot\mid s)-\mu(\cdot\mid s)
=
g_s\bigl(\pi(\cdot\mid s)-\mu(\cdot\mid s)\bigr).
\]
Therefore
\begin{align*}
\DTV(\mu(\cdot\mid s),\pi_g(\cdot\mid s))
&=
\frac12
\left\|
\pi_g(\cdot\mid s)-\mu(\cdot\mid s)
\right\|_1\\
&=
g_s\,\DTV(\mu(\cdot\mid s),\pi(\cdot\mid s)).
\end{align*}
Applying the token-level threshold and prefix-budget constraints to
$\pi_g$ therefore recovers an exact-TV guarantee with $D_t$
replaced by $g_tD_t$. We do not deploy $\pi_g$ in our experiments;
the construction is included to clarify what a soft-gate
\emph{guarantee} would require, and to distinguish it from the
effective-gradient interpretation above.

\subsection{Default gate and SAPO compatibility}

We use $g_{\mathrm{inv}}(x)=\min\{1,1/x\}$, which trivially satisfies
$xg(x)\le 1$ and $g(x)=1$ for $x\le 1$. The theorem chain only
requires the admissibility condition $xg(x)\le 1$, so any other
admissible gate can be substituted without changing the formal
guarantee. Soft gates introduced by \SAPO{}-style
schemes~\citep{gao2025soft} can be plugged in by applying
their attenuation function to the normalized
\CPPO{} score $x_t^{\CPPOSHORT{}}$ in place of a sampled-ratio score.

\section{Experiment details and per-benchmark breakdown}
\label{app:exp-details}

\noindent\textbf{Training stack.}
All runs use the \texttt{verl}-compatible GRPO/DAPO trainer with
group-normalized advantages, the \texttt{mask\_std\_0} filter (prompts
with zero advantage variance are skipped), no entropy regularizer, and
no KL-to-reference penalty in the loss. Optimizer is AdamW at a maximum
learning rate of $1\times 10^{-6}$.

\noindent\textbf{Rollout and update budget.}
At each iteration we draw a \emph{rollout batch} of prompts and unroll
\texttt{rollout.n} responses per prompt under the current policy
$\mu$, then split each rollout batch into \texttt{ministep} gradient
\emph{minibatches} of size \texttt{train\_bs} for the policy update.
Training rollouts are sampled from the actor's untruncated softmax;
validation rollouts use temperature $0.7$ and top-$p=0.95$ with $16$
samples per prompt. Table~\ref{tab:exp-config} reports the
per-model rollout / minibatch configuration
(\texttt{ministep}~$=\texttt{rollout\_bs}/\texttt{train\_bs}$).

\begin{table}[!ht]
\centering
\caption{Per-model rollout, update, and trust-region configuration.
``Prompts'' is the prompt count per rollout iteration, $n$ is the
number of responses per prompt, and ``updates'' is the number of
gradient minibatches per rollout iteration. The token-level threshold scale
$\delta$ is shared by \DPPO{} and \CPPO{}; the listed $\Db^{\min}$
values are the minimum values used by the Base-model warm-up
calibration, whose upper bound is dynamically bounded at $2\Db^{\min}$;
$\wmin$ is the weight
floor.}
\label{tab:exp-config}
\small
\setlength{\tabcolsep}{4pt}
\begin{tabular*}{\textwidth}{@{\extracolsep{\fill}}lcccccccc@{}}
\toprule
Model                  & Prompts & $n$ & Minibatch & Updates & $T_{\max}$ & $\delta$ & $\Db^{\min}$ & $\wmin$ \\
\midrule
Qwen3-1.7B             & 64                   & 8                  & 32                 & 2                 & 8\,k       & 0.15 & 0.015 & 0.8 \\
Qwen3-1.7B-Base        & 64                   & 8                  & 32                 & 2                 & 8\,k       & 0.15 & 0.020 & 0.8 \\
Qwen3-8B-Base          & 128                  & 8                  & 32                 & 4                 & 8\,k       & 0.15 & 0.020 & 0.8 \\
Qwen3-30B-A3B-Base     & 256                  & 16                 & 32                 & 8                 & 16\,k      & 0.20 & 0.020 & 0.8 \\
\bottomrule
\end{tabular*}
\end{table}

\noindent\textbf{Evaluation.}
We evaluate every checkpoint on AIME24, AIME25, and
AIME26. For each benchmark we report
$\mathrm{Avg}@16$ --- the success rate
computed over $16$ sampled completions per prompt under the
validation decoding configuration above. As a single summary
score we further report
\textbf{AIME24/25/26 $\mathrm{Avg}@16$}, the unweighted mean of the
three per-benchmark $\mathrm{Avg}@16$ values. The best value
reported in Tables~\ref{tab:main} and~\ref{tab:main-full} is the
highest AIME24/25/26 $\mathrm{Avg}@16$ attained on the validation
curve within $[0,T^{\mathrm{stop}}]$, where $T^{\mathrm{stop}}$ is
the matched evaluation horizon used for the training-curve plots;
the per-benchmark numbers in Table~\ref{tab:main-full} are the
components of the best mean (i.e.\ the per-benchmark scores at
the step that attained the best mean), \emph{not} the
per-benchmark maxima.

\noindent\textbf{Baselines and divergence scores.}
We group the baselines by the signal used to decide whether a
token update that moves the sampled-token ratio farther from one
remains inside the trust region. \GRPO{} and
\CISPO{} are sampled-ratio baselines: they operate directly
on the sampled importance ratio $\rho_t$ (with asymmetric
clip thresholds for \CISPO{}) and do not use a distributional
divergence score. The \GRPO{} row uses the Clip-Higher recipe, with
$(\epsilon_{\mathrm{low}},\epsilon_{\mathrm{high}})=(0.2,0.28)$
in the asymmetric clipping interval
$[1-\epsilon_{\mathrm{low}},1+\epsilon_{\mathrm{high}}]$; we denote
it as \GRPO{} in the tables for brevity. The remaining baseline
hyperparameters use the values recommended in their original
papers. \MinPro{} is a prefix-ratio baseline that
weights the sampled importance ratio by a non-cumulative
minimum-prefix-ratio surrogate, again without a divergence
score. \DPPO{} applies a uniform reduced-TV token-level mask with
the per-model token-level threshold scale $\delta$ in Table~\ref{tab:exp-config}.
\CPPO{} uses this matched scale inside the weighted and prefix
constraints. For all token-level
divergence masks in the main experiments, we estimate $D_t$
with a Top-$K$ ($K{=}20$) reduced-TV score. Following the
\DPPO{} approximation construction, this score keeps the top-$K$
rollout-policy tokens together with the sampled token and an
``other'' category, yielding a partitioned-TV lower bound on the
exact full-vocabulary TV. \DPPO{} and \CPPO{} use a matched Top-$K$
score and per-model threshold scale $\delta$, so the empirical
comparison isolates the prefix constraints rather than the token-level
divergence approximation. This reduced-TV lower bound is used for the
matched empirical mask, not as the exact-TV guarantee in
Theorem~\ref{thm:pact-main}. In Base-model runs, the listed
$\Db$ is the minimum value in the warm-up calibration described in
Section~\ref{sec:exp-setup}.
Specifically, for a rollout sequence with token-level divergences
$D_{1:T}$, the dynamic sequence budget is computed as
$\delta_b^{\text{seq}} =
\text{clamp}(\text{Quantile}(D_{1:T}, 0.9), \Db^{\min}, 2\Db^{\min})$,
where $\text{Quantile}(D_{1:T}, 0.9)$ represents the boundary value
above which the top $10\%$ largest divergences lie. Hence, the
$\delta_b = 0.03$ ablation corresponds to the setting where this
quantile-based calibration is disabled and $\delta_b$ is fixed
instead.
The Binary-TV
score is used only in the right panel of
Figure~\ref{fig:ablation-base}.
\TRM{}-Max and \TRM{}-Avg use KL
as the token-level divergence in their original form:
$M(y)=\ind[\max_t D^{\mathrm{KL}}_t\le\delta]$ and
$M(y)=\ind[T^{-1}\sum_t D^{\mathrm{KL}}_t\le\delta_{\mathrm{avg}}]$
respectively, with
$D^{\mathrm{KL}}_t=\mathrm{KL}\!\bigl(\mu(\cdot\mid s_t)\,\|\,\pi(\cdot\mid s_t)\bigr)$.
The comparison of Section~\ref{sec:theory}
(\TRM{}-Max $\leftrightarrow$ \CPPO{}'s uniform token-level threshold;
\TRM{}-Avg $\leftrightarrow$ terminal-only specialization of
\CPPO{}'s prefix constraints) does not depend on the choice
between TV and KL and applies to either. Each baseline's
corresponding threshold follows the value recommended in its
original paper.

\noindent\textbf{Batched mask computation.}
Algorithm~\ref{alg:pact} gives the readable one-response version.
The implementation applies the same recurrence to a batch tensor.
For a valid-token mask $M\in\{0,1\}^{B\times T}$, padding positions
are zeroed in the weights and in $Z=w\odot D$. A cumulative sum
along the response dimension gives the tensors $W$ and $S$ for all
positions at once, and a one-token right shift gives $W^{-}$ and
$S^{-}$ so that the threshold at token $t$ uses only the preceding
prefix. The final binary mask is then multiplied into the same
token-level ratio--advantage terms as the base \PPO{}-style clipped
objective.

\noindent\textbf{Base-model warm-up diagnostics.}
Figure~\ref{fig:base-deltab-diagnostics} visualizes the warm-up
calibration used by the three Base-model \CPPO{} runs.
During the initial exploration phase of Base-model training, the
average token-level divergence is initially large but decays rapidly
after a few steps. To prevent prematurely clipping these early
exploratory tokens before training stabilizes, a uniform
sequence-independent threshold can be overly restrictive. By utilizing
the 90th percentile of the sequence's own divergence profile, the
constraint dynamically adapts to the current sequence's statistics:
\begin{equation}
\delta_b^{\text{seq}} = \min\left(2\Db^{\min}, \max\left(\Db^{\min}, P_{90}(D_{1:T})\right)\right)
\end{equation}
where $P_{90}$ denotes the 90th percentile of the token-level
divergence sequence $D_{1:T}$. The top row of
Figure~\ref{fig:base-deltab-diagnostics} shows the mean effective
prefix-average threshold after calibration; the $\Db$ values in
Table~\ref{tab:exp-config} are the minimum values $\Db^{\min}$.
The bottom row shows the fraction of masked tokens rejected by the
prefix-budget condition.
The effective values remain well below the corresponding
token-level thresholds ($\delta=0.15$ or $0.20$), and the prefix-budget
mask fraction is concentrated early in training, which is the regime
where the average token-level divergence is exceptionally large
before rapidly decaying as the policy stabilizes.

\begin{figure}[!ht]
\centering
\includegraphics[width=\textwidth]{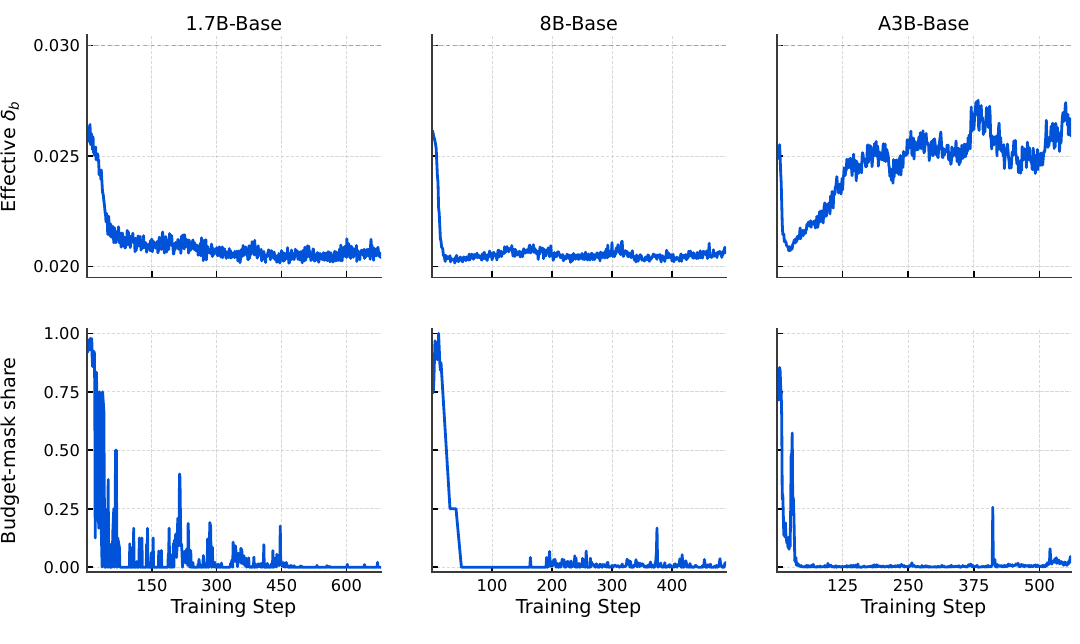}
\caption{Warm-up diagnostics for the three Base-model \CPPO{} runs.
Each column is one model. Top: mean effective $\Db$ after the
per-sequence warm-up calibration. Bottom: fraction of masked tokens
rejected by the prefix-budget condition.}
\label{fig:base-deltab-diagnostics}
\end{figure}

\FloatBarrier

\noindent\textbf{Full training diagnostics.}
Figures~\ref{fig:app-complete-1p7b}--\ref{fig:app-complete-a3b-base}
provide per-model training traces behind the aggregate results in
Table~\ref{tab:main}. The top row separates the three validation
components, AIME24, AIME25, and AIME26, while the bottom row reports
training reward, response length, and the relative log-probability
error between the rollout and training log-probabilities. These plots
are included as diagnostics rather than as additional selection rules:
the reported scores in Tables~\ref{tab:main} and~\ref{tab:main-full}
are still selected only by the best AIME24/25/26 mean within the
matched evaluation horizon. The relative log-probability error panels
are implementation stability checks and are not part of the training
objective or the formal guarantee.

\begin{figure}[!htbp]
\centering
\includegraphics[width=\textwidth]{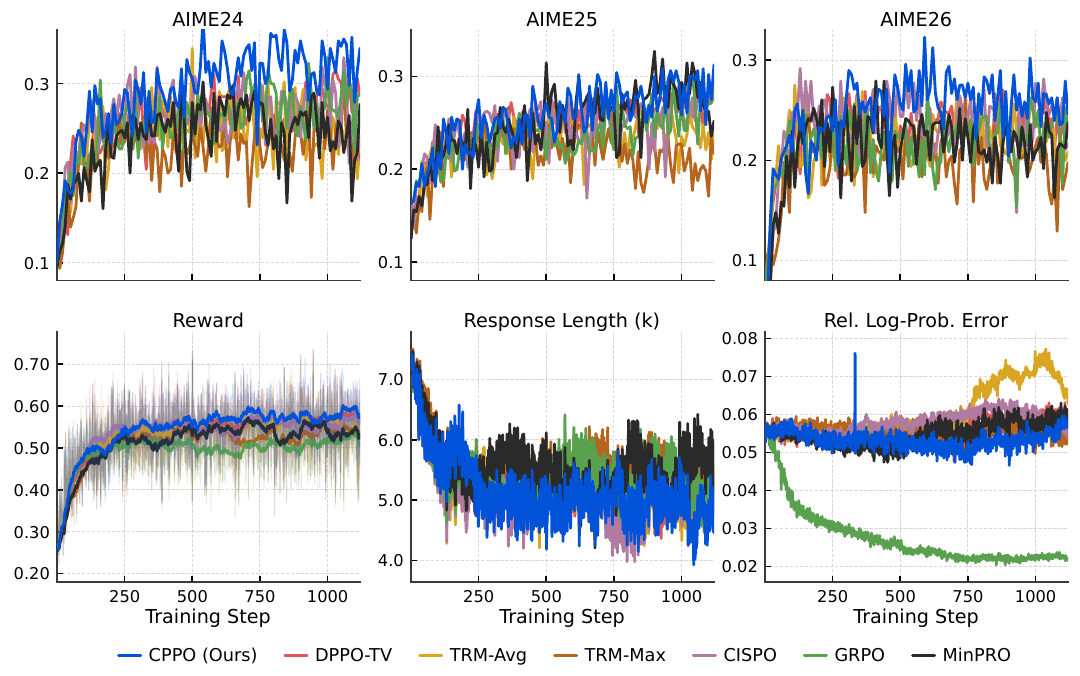}
\caption{Complete training diagnostics for Qwen3-1.7B
(post-trained). Top: AIME24, AIME25, and AIME26 validation
$\mathrm{Avg}@16$. Bottom: training reward, response length, and
relative log-probability error. Only reward is smoothed for
readability.}
\label{fig:app-complete-1p7b}
\end{figure}

\begin{figure}[!htbp]
\centering
\includegraphics[width=\textwidth]{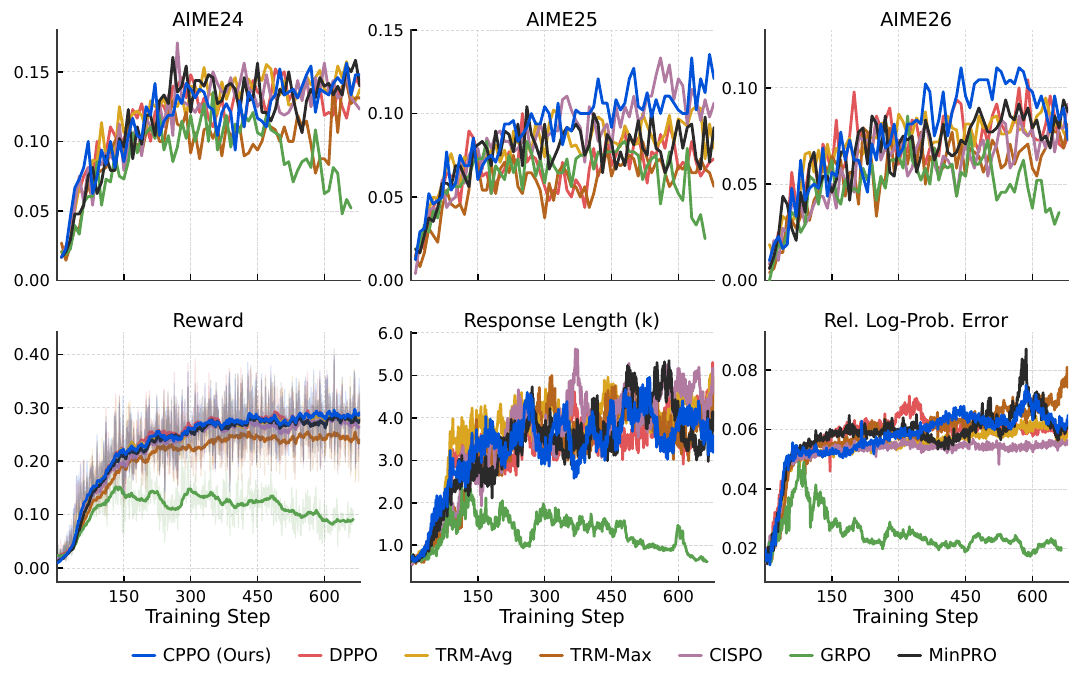}
\caption{Complete training diagnostics for Qwen3-1.7B-Base.
Top: AIME24, AIME25, and AIME26 validation $\mathrm{Avg}@16$.
Bottom: training reward, response length, and relative
log-probability error. Only reward is smoothed for readability.}
\label{fig:app-complete-1p7b-base}
\end{figure}

\begin{figure}[!htbp]
\centering
\includegraphics[width=\textwidth]{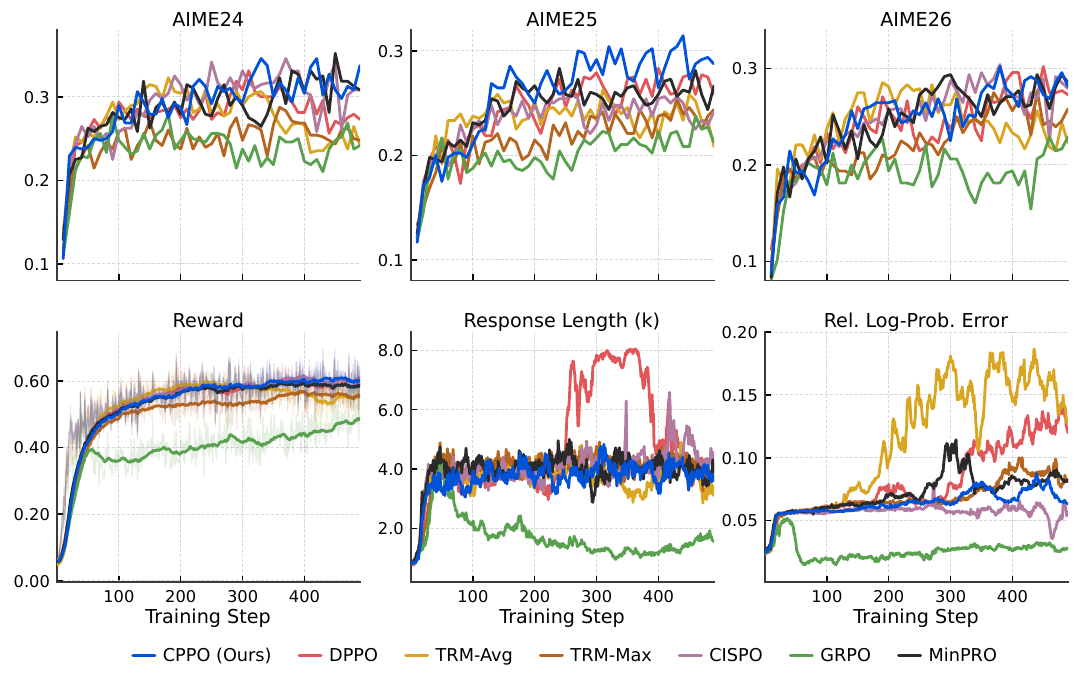}
\caption{Complete training diagnostics for Qwen3-8B-Base.
Top: AIME24, AIME25, and AIME26 validation $\mathrm{Avg}@16$.
Bottom: training reward, response length, and relative
log-probability error. Only reward is smoothed for readability.}
\label{fig:app-complete-8b-base}
\end{figure}

\begin{figure}[!htbp]
\centering
\includegraphics[width=\textwidth]{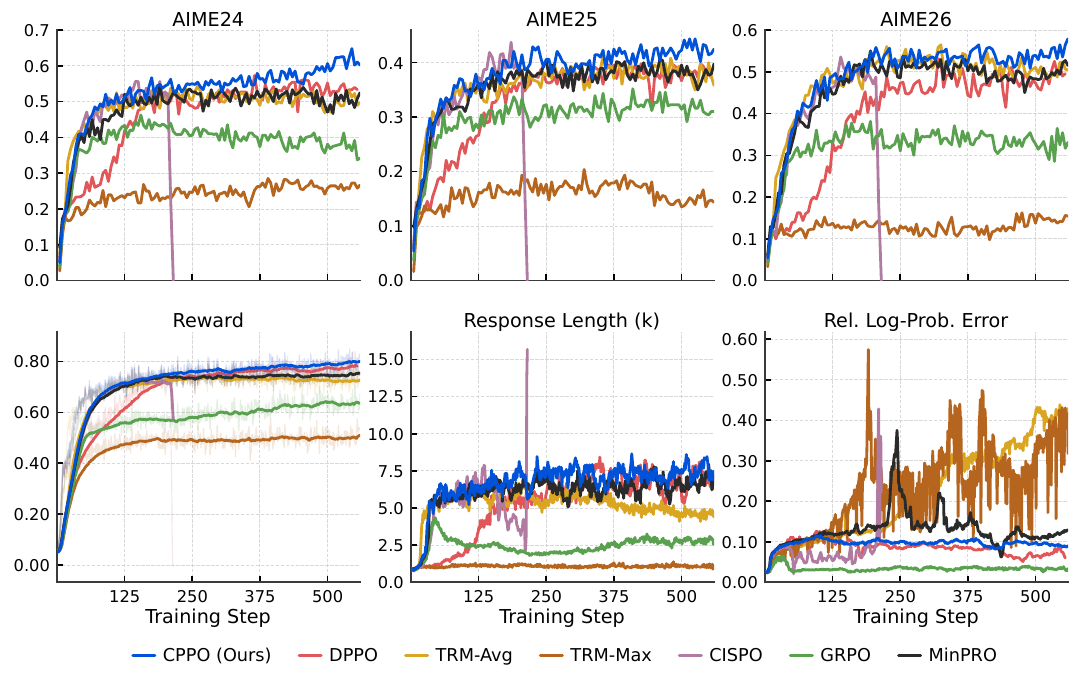}
\caption{Complete training diagnostics for Qwen3-30B-A3B-Base.
Top: AIME24, AIME25, and AIME26 validation $\mathrm{Avg}@16$.
Bottom: training reward, response length, and relative
log-probability error. Only reward is smoothed for readability.}
\label{fig:app-complete-a3b-base}
\end{figure}

\FloatBarrier

\begin{table}[!ht]
\centering
\caption{\textbf{Detailed per-benchmark evaluation results.} This table
expands the aggregate results from Table~\ref{tab:main}
(\%, $\mathrm{Avg}@16$). \textbf{AVG} denotes the best AIME24/25/26
$\mathrm{Avg}@16$ within the matched evaluation window; the
per-benchmark columns report the AIME24/AIME25/AIME26 scores at this
best-average checkpoint, not their individual maxima.
\textit{collapse} indicates training divergence, which only occurred
for \CISPO{} on Qwen3-30B-A3B-Base at step~215. Within each model
block, the best score in a column is in \textbf{bold} and the
second-best is \underline{underlined}.}
\label{tab:main-full}
\small
\resulttablespacing
\begin{tabularx}{\textwidth}{l *{4}{C}}
\toprule
\textbf{Method} & \textbf{AIME24} & \textbf{AIME25} & \textbf{AIME26} & \textbf{AVG} ($\uparrow$) \\
\midrule
\rowcolor{grouphl}
\multicolumn{5}{l}{\textit{Qwen3-1.7B} \quad ($T^{\mathrm{stop}}=1120$)}\\
\GRPO{}                & 30.41                       & 28.75                       & 24.58                       & 27.91             \\
\MinPro{}              & 26.87                       & \tbest{32.71}               & 23.54                       & 27.71             \\
\CISPO{}               & 30.62                       & 28.12                       & \tsecond{27.71}             & \tsecond{28.82}   \\
\DPPO{}                & \tsecond{31.04}             & 28.22                       & 25.31                       & 28.19             \\
\TRM{}-Max             & 26.46                       & 22.91                       & 26.25                       & 25.21             \\
\TRM{}-Avg             & 29.79                       & 25.41                       & 25.41                       & 26.87             \\
\rowcolor{ourshl}
\CPPO{} (ours)         & \tbest{34.79}               & \tsecond{30.63}             & \tbest{30.21}               & \tbest{31.88}     \\
\midrule
\rowcolor{grouphl}
\multicolumn{5}{l}{\textit{Qwen3-1.7B-Base} \quad ($T^{\mathrm{stop}}=680$)}\\
\GRPO{}                & 11.25                       & \phantom{0}6.87             & \phantom{0}8.54             & \phantom{0}8.89   \\
\MinPro{}              & \tbest{16.04}               & 10.41                       & \phantom{0}6.66             & 11.04             \\
\CISPO{}               & \tsecond{15.62}             & \tsecond{12.08}             & \phantom{0}7.91             & \tsecond{11.87}   \\
\DPPO{}                & 13.96                       & \phantom{0}8.33             & \tsecond{10.41}             & 10.90             \\
\TRM{}-Max             & 14.16                       & \phantom{0}7.29             & \phantom{0}7.71             & \phantom{0}9.72   \\
\TRM{}-Avg             & 15.41                       & 10.21                       & \phantom{0}9.48             & 11.70             \\
\rowcolor{ourshl}
\CPPO{} (ours)         & 15.21                       & \tbest{12.71}               & \tbest{10.42}               & \tbest{12.78}     \\
\midrule
\rowcolor{grouphl}
\multicolumn{5}{l}{\textit{Qwen3-8B-Base} \quad ($T^{\mathrm{stop}}=490$)}\\
\GRPO{}                & 25.21                       & 23.75                       & 22.91                       & 23.96             \\
\MinPro{}              & \tbest{35.21}               & 26.04                       & 27.91                       & \tsecond{29.72}   \\
\CISPO{}               & 32.91                       & 25.41                       & \tbest{30.41}               & 29.58             \\
\DPPO{}                & 30.21                       & \tsecond{26.87}             & \tsecond{29.58}             & 28.89             \\
\TRM{}-Max             & 28.33                       & 24.37                       & 27.50                       & 26.73             \\
\TRM{}-Avg             & 31.46                       & 25.00                       & 27.50                       & 27.98             \\
\rowcolor{ourshl}
\CPPO{} (ours)         & \tsecond{34.58}             & \tbest{30.00}               & 28.75                       & \tbest{31.11}     \\
\midrule
\rowcolor{grouphl}
\multicolumn{5}{l}{\textit{Qwen3-30B-A3B-Base} \quad ($T^{\mathrm{stop}}=560$)}\\
\GRPO{}                & 43.75                       & 32.71                       & 38.12                       & 38.19             \\
\MinPro{}              & 52.91                       & 38.54                       & 52.91                       & 48.12             \\
\CISPO{}               & \textit{collapse}           & \textit{collapse}           & \textit{collapse}           & \textit{collapse} \\
\DPPO{}                & \tsecond{57.29}             & \tsecond{39.58}             & 50.83                       & \tsecond{49.23}   \\
\TRM{}-Max             & 26.46                       & 20.41                       & 13.96                       & 20.27             \\
\TRM{}-Avg             & 54.37                       & 37.08                       & \tsecond{55.41}             & 48.96             \\
\rowcolor{ourshl}
\CPPO{} (ours)         & \tbest{64.79}               & \tbest{43.13}               & \tbest{56.46}               & \tbest{54.79}     \\
\bottomrule
\end{tabularx}
\end{table}

\FloatBarrier

\end{document}